\documentclass[letterpaper]{article} 
\usepackage{aaai2026}  
\usepackage{times}  
\usepackage{helvet}  
\usepackage{courier}  
\usepackage[hyphens]{url}  
\usepackage{graphicx} 
\usepackage{natbib}  
\usepackage{caption} 
\usepackage{algorithm}
\usepackage{algorithmic}

\usepackage{soul}
\usepackage{url}
\usepackage{amsmath}
\usepackage{amsthm}
\usepackage{booktabs}
\usepackage{subfig}
\usepackage{multirow}
\usepackage{bm}
\usepackage{amssymb}
\usepackage{enumitem}
\usepackage{pifont}

\newtheorem{theorem}{Theorem}

\newtheorem{assumption}{Assumption}
\usepackage{newfloat}
\usepackage{listings}
\DeclareCaptionStyle{ruled}{labelfont=normalfont,labelsep=colon,strut=off} 
\floatstyle{ruled}
\newfloat{listing}{tb}{lst}{}
\floatname{listing}{Listing}
\title{Towards OOD Generalization in Dynamic Graphs via Causal Invariant Learning}
\author{
    Xinxun Zhang\textsuperscript{\rm 1,2,3}, Pengfei Jiao\textsuperscript{\rm 1,2}, Mengzhou Gao\textsuperscript{\rm 1}\thanks{Corresponding author}, Tianpeng Li\textsuperscript{\rm 4}, Xuan Guo\textsuperscript{\rm 4}
}
\affiliations{
    \textsuperscript{\rm 1}School of Cyberspace, Hangzhou Dianzi University, China\\
    \textsuperscript{\rm 2}State Key Laboratory of AI Safety, Beijing\\
    \textsuperscript{\rm 3}School of Computer Science, Wuhan University, China\\
    \textsuperscript{\rm 4}College of Intelligence and Computing, Tianjin University, China

   xxzhangstu@whu.edu.cn, \{pjiao, mzgao\}@hdu.edu.cn, \{ltpnimeia, guoxuan\}@tju.edu.cn,
}

\usepackage{bibentry}

\begin{document}

\maketitle

\begin{abstract}
Although dynamic graph neural networks (DyGNNs) have demonstrated promising capabilities, most existing methods ignore out-of-distribution (OOD) shifts that commonly exist in dynamic graphs. Dynamic graph OOD generalization is non-trivial due to the following challenges: 1) Identifying invariant and variant patterns amid complex graph evolution, 2) Capturing the intrinsic evolution rationale from these patterns, and 3) Ensuring model generalization across diverse OOD shifts despite limited data distribution observations. Although several attempts have been made to tackle these challenges, none has successfully addressed all three simultaneously, and they face various limitations in complex OOD scenarios. To solve these issues, we propose a \textbf{Dy}namic graph \textbf{C}ausal \textbf{I}nvariant \textbf{L}earning (DyCIL) model for OOD generalization via exploiting invariant spatio-temporal patterns from a causal view. Specifically, we first develop a dynamic causal subgraph generator to identify causal dynamic subgraphs explicitly. Next, we design a causal-aware spatio-temporal attention module to extract the intrinsic evolution rationale behind invariant patterns. Finally, we further introduce an adaptive environment generator to capture the underlying dynamics of distributional shifts. Extensive experiments on both real-world and synthetic dynamic graph datasets demonstrate the superiority of our model over state-of-the-art baselines in handling OOD shifts. 
\end{abstract}


\section{Introduction}

Dynamic graphs are ubiquitous across various domains \cite{yang2020relation,zhao2019t,ijcai2025p347}. Many dynamic graph neural networks (DyGNNs) methods have been proposed to learn meaningful node embedding, achieving remarkable success in various downstream tasks \cite{huang2022dlp,huang2023TGB,wu2024feasibility,zhang2024tifs}. Despite their great progress, most existing methods \cite{seo2018structured,pareja2020evolvegcn,sankar2020dysat,zhong2024efficient} are based on the data assumption of independent identically distributed (IID), where the training and test data share the same distribution. However, in real-world dynamic graphs, distribution shifts are common due to the confounding environmental biases in their generation process. Most DyGNNs don't consider the effect of confounding biases from unstable and variant environments, limiting their ability to capture invariant spatio-temporal patterns and generalize to dynamic graphs with out-of-distribution (OOD) shifts. 



To tackle this issue, and inspired by recent research in invariant learning \cite{arjovsky2019invariant,krueger2021VREx,wu2022dir}, we argue that the causal patterns in dynamic graphs to the labels remain stable and invariant across different environments. Therefore, we study generalized dynamic graph representation learning under OOD shifts by effectively capturing invariant spatio-temporal patterns from a causal view. However, it is non-trivial due to the following key challenges: (1) Identification: How can we distinguish between invariant and variant patterns in dynamic graphs amid complex evolution? (2) Extraction: How can we extract the intrinsic stable evolution rationale from the identified patterns? (3) Generalization: How can we enable the model to generalize to a wide range of unseen data with distribution shifts in the presence of limited observable data?

Several recent attempts have been made to address these challenges \cite{zhang2022dynamic,zhang2023spectral,yuan2023eagle,yang2024improving, tieu2025oodlinker}. However, they suffer from various issues and limitations. First, DIDA \cite{zhang2022dynamic}, SILD \cite{zhang2023spectral}, and EAGLE \cite{yuan2023eagle} fail to identify causal dynamic subgraphs or explicitly learn invariant patterns, which are crucial for understanding the mechanisms behind OOD generalization in dynamic graphs and providing better interpretability. In addition, DIDA and SILD extract evolution rationales using the entire dynamic graph, which makes the evolution rationales that they extract unstable since the dynamic graph contains environmental biases. On the other hand, EAGLE and EpoD \cite{yang2024improving} overlook the evolutionary rationales of invariant patterns. Lastly, DIDA, SILD, and EpoD generate discrete and limited environment instances by using observable graph snapshots for interventions, while OOD-linker \cite{tieu2025oodlinker} overlooks the role of environmental factors in dynamic graphs. Both limitations weaken their generalization ability in more challenging dynamic OOD scenarios. EAGLE enhances its generalization capability by inferring environment distributions to generate sufficient instances, but it requires additional labels (a mixture of time index and environment index), which limits its scalability in various dynamic OOD scenarios. 



To address the limitations of these methods while simultaneously tackling the three challenges as a whole, we first construct a structural causal model (SCM) based on the generation process of dynamic graphs to explore the causal effects among various factors. Then, we instantiate a novel \textbf{Dy}namic graph \textbf{C}ausal \textbf{I}nvariant \textbf{L}earning (DyCIL) model based on the causal relationships to handle OOD shifts in dynamic graphs via exploiting invariant spatio-temporal patterns from a causal view. DyCIL can effectively capture the invariant spatio-temporal patterns to label remain stable across various latent environments through jointly optimizing three mutually promoting modules. Specifically, we propose a dynamic causal subgraph generator to explicitly identify causal dynamic subgraphs, offering deeper insights and enhanced interpretability into the underlying mechanisms of dynamic graphs OOD generalization. Then, we design a causal-aware spatio-temporal attention module for learning representations capable of OOD generalization under distribution shifts by exploiting the intrinsic evolution rationale of causal dynamic subgraphs. This endows DyCIL with a higher capability to learn invariant spatio-temporal patterns. Lastly, we introduce an adaptive environment generator that adaptively infers environment distribution without requiring additional information. This enables DyCIL to generate ample continuous environment instances to capture the underlying dynamics of distributional shifts, thereby improving the model's generalization ability across complex OOD scenarios. Contributions of our work are summarized as follows: 
\begin{itemize}[leftmargin=*]
    \item We introduce DyCIL, a novel model designed for OOD generalization in dynamic graphs from a causal perspective. DyCIL can learn invariant spatio-temporal patterns and generalize well in dynamic scenarios with various OOD shifts.
    \item We develop multiple effective modules enabling DyCIL to handle OOD shifts, including a dynamic causal subgraph generator that identifies causal dynamic subgraphs, a causal-aware spatio-temporal attention module, and an adaptive environment generator to enable the model to learn generalized invariant representations under various OOD shifts.
    \item Extensive experiments on both real-world datasets and synthetic datasets demonstrate the superiority of our proposed model over state-of-the-art methods on different dynamic graph prediction tasks with OOD shifts.
\end{itemize}

\section{Related Work}

DyGNNs model evolving structures and temporal dependencies in real-world systems \cite{huang2022dlp,huang2023TGB,ijcai2025p1166}. Existing methods include continuous-time models \cite{nguyen2018continuous,jin2022neural,rossi2020temporal} that process event streams with time encodings or memory modules, and discrete-time models \cite{sankar2020dysat,pareja2020evolvegcn,zhang2023dyted} that represent dynamic graphs as timestamped snapshots using GNNs and sequence models. While effective for temporal representation learning, they overlook OOD shift. To enhance robustness, graph OOD generalization has been explored through disentanglement-based methods \cite{yang2020factorizable,liu2020independence,li2022disentangled} that separate invariant and variant factors, and causality-based methods \cite{wu2022dir,li2022learning,jia2024graph} that exploit stable causal mechanisms across environments. Motivated by these advances, recent studies extend OOD generalization to dynamic graphs. DIDA \cite{zhang2022dynamic} and SILD \cite{zhang2023spectral} address OOD shifts from the perspective of decoupled learning, while EAGLE \cite{yuan2023eagle} and EpoD \cite{yang2024improving} enhance generalization by modeling latent environments. OOD-Linker \cite{tieu2025oodlinker} further improves invariance by leveraging the information bottleneck principle.

\section{Preliminary Analysis}

\subsection{A Causal View on Dynamic Graphs} \label{sec3.1}

We first take a causal look at the generating process of the dynamic graph and construct SCM \cite{pearl2000models} in Figure \ref{fig_scm}. It shows the causalities among six variables, where each directed link from one variable to another denotes a cause-effect relationship between variables. The key explanations regarding SCM are as follows:

\begin{itemize}[leftmargin=*]
    \item ${C}_t \bm{\rightarrow} \mathcal{G}_t \bm{\leftarrow} {E}_t$: The observed dynamic graph data is generated by two unobserved latent variables: the causal factor ${C}_t$ and the environment factor ${E}_t$. While ${C}_t$ denotes the genuine causal relationship, ${E}_t$ is the unstable variant patterns arising from changes in external variables over time. For example, living and workplace will affect a node's future status in social networks. 
    
    \item ${S}_t \bm{\rightarrow}{C}_t \bm{\leftarrow} {T}_{t^{'}}$: For a dynamic graph, the causal factor ${C}_t$ consist of spatial factor ${S}_t$ and temporal factor ${T}_{t'}$, where ${S}_t$ denotes the topological information at timestamp $t$ and ${T}_{t^{'}}(t' < t)$ represents the temporal evolution information from historical snapshots.

    \item ${C}_t \bm{\rightarrow} \mathcal{G}_t \bm{\rightarrow} {Y}_t$: This link indicates that the causal factor ${C}_t$ is the only endogenous parent to determine the generation of label ${Y}_t$ of $\mathcal{G}_t$ in dynamic graphs. 
    
    \item ${E}_t \bm{\rightarrow} {C}_t$: This link means that a spurious correlation exists between causal factor ${C}_t$ and environment factor ${E}_t$, which usually changes over time in dynamic graphs. 
    
\end{itemize}

According to the SCM of dynamic graphs, we recognize a backdoor path between $C_t$ and $Y_t$, \textit{i.e.}, ${C}_t \bm{\leftarrow} {E}_t \bm{\rightarrow} \mathcal{G}_t \bm{\rightarrow} {Y}_t$, which will cause a misleading correlation between ${C}_t$ and ${Y}_t$. Furthermore, dynamic graphs collected from different environments exhibit various confounding biases. 
The spatio-temporal patterns extracted using confounding biases are unstable, leading to the failure of most DyGNNs to generalize to OOD shift scenarios.


\begin{figure}[tbp]
\centering
\includegraphics[width=0.8\columnwidth]{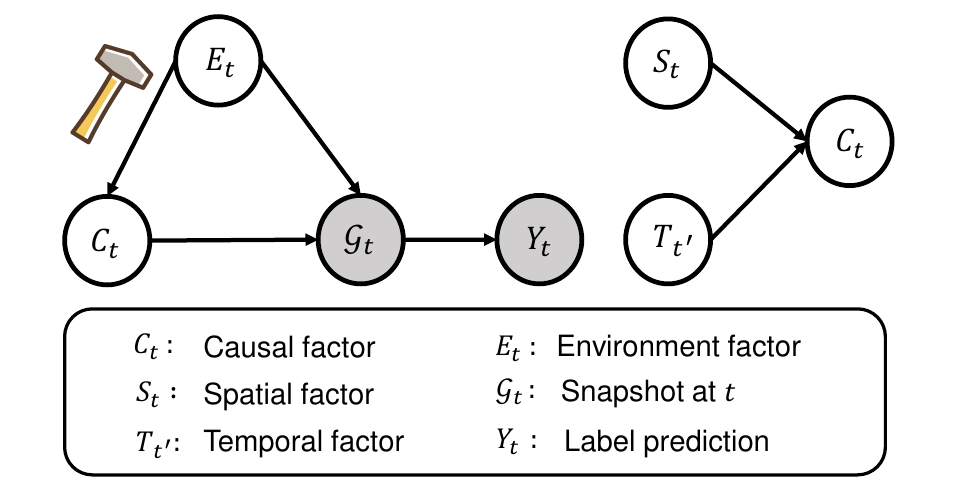}
\caption{Structural causal model for the generating process of the dynamic graph, where grey and white variables represent observed variables and unobserved, respectively.}
\label{fig_scm}    
\end{figure}

\subsection{Backdoor Adjustment} To ensure that predictions in dynamic graphs are based solely on causal factor $C_t$, it is imperative to tackle the confounding bias effect from the environment factor ${E}_t$. To achieve it, we utilize causal tools \textit{do-calculus} on the variable ${C}_t$ to eliminate the backdoor path between ${C}_t$ and ${Y}_t$, a process commonly referred to as backdoor adjustment \cite{pearl2014interpretation}. By doing so, the probability distribution $P({Y}_t|do({C}_t)$ can be estimated without interference from ${E}_t$. Specifically, we can estimate $P({Y}_t|do({C}_t))$ by stratifying ${E}_t$ between ${C}_t$ and ${Y}_t$,
\begin{equation}
    P({Y}_t|do({C}_t)) = \sum_{e_{i} \in {E}_t} P({Y}_t|{S}_t, {T}_{t'}, {E}_{t}=e_{i}) P({E}_{t} =e_{i})
\end{equation}
where $e_i$ denotes the discrete component of environment factor $E_t$. Detailed derivations are provided in Appendix \ref{app_causal}. 

\section{Model Instantiations} \label{model}
A dynamic graph can be denoted as a series of snapshots $\mathcal{G}_{1:T}=\left\{\mathcal{G}_1, \mathcal{G}_2,\cdots, \mathcal{G}_T\right\}$, where $T$ is the total number of snapshots. $\mathcal{G}_t$ denotes the snapshot at timestamp $t$, which is a graph with a node-set $\mathcal{V}_t$, an edge-set $\mathcal{E}_t$ and feature matrix $ \mathbf{X}_t$, \textit{i.e.}, $\mathcal{G}_t=(\mathcal{V}_t, \mathcal{E}_t,  \mathbf{X}_t)$. We denote $\mathbf{A}_t$ as the adjacency matrix of $\mathcal{G}_t$.In this paper, we mainly focus on dynamic prediction tasks, which aim to utilize previous snapshot sequences to predict future states. \textit{i.e.}, $P({Y}_{t}|\mathcal{G}_{1}, \mathcal{G}_2,\cdots, \mathcal{G}_t)$, where ${Y}_{t}$ denote the node properties or interactive relationship at timestamp ${t+1}$. To tackle the OOD generalization challenge in dynamic graphs, we propose a model named DyCIL, which implements the aforementioned backdoor adjustment mechanism. The overall architecture of DyCIL is illustrated in Figure~\ref{fig_model}.

\subsection{Dynamic Causal Subgraph Generator}
Based on the above analysis, we know that dynamic graphs are generated from both causal factors and environmental factors. Therefore, we assume that each dynamic graph $\mathcal{G}_{1:T}$ has a causal dynamic subgraph $\mathcal{G}^{c}_{1:T}=\left\{\mathcal{G}_1^c, \mathcal{G}_2^c,\cdots, \mathcal{G}_T^c\right\}$ whose correlation with the label is invariant across different environments. We denote the complement of the causal dynamic subgraph as the variant dynamic subgraph $\mathcal{G}^e_{1:T}= \mathcal{G}_{1:t} \setminus \mathcal{G}_{1:t}^{c} = \left\{\mathcal{G}_1^e, \mathcal{G}_2^e,\cdots, \mathcal{G}_T^e\right\}$, whose relationship with the label is variant across different environments.

We employ a generator to identify the causal dynamic subgraph $\mathcal{G}^c_{1:T} = \Psi (\mathcal{G}_{1:T})$. According to the invariant learning theory \cite{rojas2018invariant}, we make the following assumption about $\Psi (\mathcal{G}_{1:T})$,

\begin{assumption} \label{ass_1} \textit{
Given a dynamic graph, there exists a causal dynamic subgraph generator $\Psi (\mathcal{G}_{1:t})$ such that: 
(1) ${Y}_{t} = f(\Psi (\mathcal{G}_{1:t})) + \epsilon$, i.e., ${Y}_{t} \bot (\mathcal{G}_{1:t} \setminus \Psi (\mathcal{G}_{1:t})) | \Psi (\mathcal{G}_{1:t})$, where f is the predictor function, $\bot$ denotes probabilistic independence, and $\epsilon$ is random noise;
(2) $\forall \mathbf{e}_i, \mathbf{e}_j \in \mathbf{E}$, $p({Y}_{t}|\Psi (\mathcal{G}_{1:t}), \mathbf{e}_i) = p({Y}_{t}|\Psi (\mathcal{G}_{1:t}), \mathbf{e}_j)$, where $\mathbf{e}_i, \mathbf{e}_j$ denote instances randomly sampled from environment $\mathbf{E}$.}
\end{assumption}

Assumption \ref{ass_1} shows that we can develop a subgraph generator to generate causal dynamic subgraphs with sufficient and stable predictive capability for dynamic node-level prediction tasks across different environments. To characterize the optimality of causal dynamic subgraph generator, we derive the following result from Assumption~\ref{ass_1}, which formalizes the mutual information maximization property of the optimal generator.

\begin{theorem}\label{the_1}
Let $\Psi$ be a subgraph generator mapping from $\mathcal{G}_{1:t}$ to a subgraph. Under Assumption~\ref{ass_1}, the following results hold:

\begin{itemize}
    \item The optimal causal subgraph generator $\Psi^*$ satisfies:
    \begin{equation} \label{th_eq_1}
        \Psi^* = \arg\max_{\Psi} I(\Psi(\mathcal{G}_{1:t}); Y_t).
    \end{equation}
    However, directly optimizing Eq. \ref{th_eq_1} is challenging, so we address this issue by approximating it with a variational lower bound.
    \item For any generator $\Psi$, the mutual information can be lower bounded by a variational approximation:
\begin{align}
   \nonumber I(\Psi(\mathcal{G}_{1:t}); &Y_t) \\ 
    &\geq 
    \mathbb{E}_{p(\mathcal{G}_{1:t}, Y_t)} 
    \left[ 
        \log q_\phi(Y_t | \Psi(\mathcal{G}_{1:t})) 
    \right] 
    + H(Y_t)
\end{align}

    where $q_\phi$ is any variational distribution approximating $p(Y_t|\Psi(\mathcal{G}_{1:t}))$, parameterized by $\phi$, and $H(Y_t)$ is the entropy of label $Y_t$.
\end{itemize}
\end{theorem}

\begin{figure*}[htbp]
\centering
\includegraphics[width=0.95\linewidth]{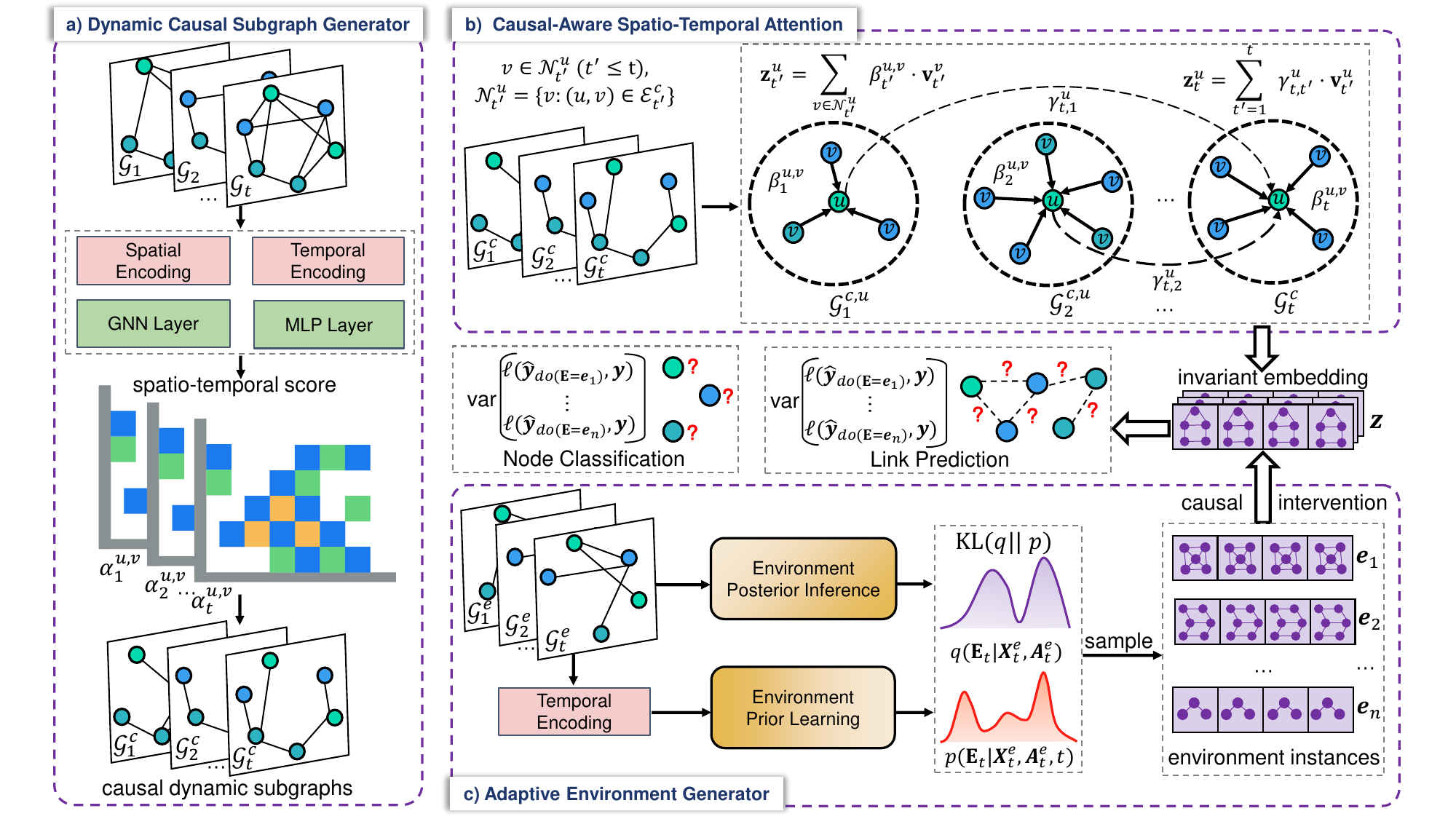} 
\caption{The framework of DyCIL.}
\label{fig_model}    
\end{figure*}

We provide the detailed proof in Appendix \ref{app_theorem}. In our work, the bound is indirectly optimized from a causal perspective through the invariance loss and the intervention loss. Therefore, $\Psi$ can be instantiated with learnable parameters to generate the causal dynamic subgraph. In static graphs, existing methods commonly utilize GNNs to generate soft masks for identifying causal subgraphs \cite{wu2022dir,li2022learning}. However, such methods assume the causal structure to be fixed, which is incompatible with dynamic graphs where both the topology and underlying causal subgraphs evolve over time. To address this, we propose to quantify the spatio-temporal importance of dynamic subgraphs by jointly encoding spatial structure and temporal dynamics, thereby enabling the generation of causal dynamic subgraphs with temporal evolution.


\textbf{Spatial and Temporal Encoding}. We first utilize the degree centrality commonly adopted in literature \cite{marwick2011see,ying2021transformers} to capture the spatial feature of nodes within the graph. Then we employ the temporal encoding technique \cite{xu2020inductive} to incorporate temporal dependency information. For node $u$ in snapshot $\mathcal{G}_t$, we have,
\begin{equation}
    \mathbf{H}_{t}^{u} = \mathbf{W}((\mathbf{X}_{t}^{u} + deg(D_{t}^{u})) \oplus TE(t)),
\end{equation}
where $\mathbf{W}$ is the trainable parameters, $deg(\cdot)$ is learnable embedding vectors specified by degree $D_{t}^{u}$ of node $u$ in timestamp $t$ as spatial encoding function, and $\oplus$ is the element-wise addition. $TE(t)$ denotes the temporal encoding function, specifically, 
\begin{equation}
TE(t) = \sqrt{\frac{1}{d}}\left[\cos \left(w_{1}t\right), \sin \left(w_{1}t\right), \cdots, \sin \left(w_{d}t\right)\right],
\end{equation}
where $w_1, \cdots, w_d$ are learnable parameters, $d$ is the encoding dimension. In this way, the spatial feature and temporal dependency information can be inherently incorporated into $\mathbf{H}_{t}^{u}$.

\textbf{Generate Causal Dynamic Subgraph}. Subsequently, we can obtain the spatio-temporal score of each edge in $\mathcal{G}_t$ with a GNN-based generator,
\begin{equation} \label{eq_edge_score}
    {\alpha}_{t}^{u,v} = {\rm MLP}(\mathbf{h}_{t}^{u}||\mathbf{h}_{t}^{v}), (u, v) \in \mathcal{E}_t, \hspace{1em}
    \mathbf{h}_{t} = {\rm GNN}(\mathcal{G}_t, \mathbf{H}_t),
\end{equation} 
where $|| $ denotes concatenation operation, $\mathbf{h}_{t} \in \mathbb{R}^{\mathcal{|V}_t| \times |d|} $ summarizes the \textit{d}-dimensional spatio-temporal matrix of all nodes in $\mathcal{G}_t$, $\mathbf{h}_t^{u}$ and $\mathbf{h}_t^{v}$ denotes the vector of node \textit{u} and \textit{v} in $\mathbf{h}_{t}$, and ${\alpha}_{t}^{u,v}$ denotes the spatio-temporal score of edge between node $u$ and node $v$ in $\mathcal{G}_t$. Then we will select the edges with the highest spatio-temporal score to construct a causal dynamic subgraph $\mathcal{G}_t^c$ and collect the complement of $\mathcal{G}_t^c$ as variant dynamic subgraph $\mathcal{G}_t^e$, particularly,
\begin{equation} \label{eq_subgraph}
    \mathcal{E}_{t}^{c} = {\rm Top}_{r}({\alpha}_{t}), \hspace{1em}
    \mathcal{E}_{t}^{e} = {\rm Top}_{1-r}(1- {\alpha}_{t}),
\end{equation} 
where $\alpha_t$ denotes the spatio-temporal scores of all edges in $\mathcal{G}_t$, and ${\rm Top}_r(\cdot)$ selects the ${\rm top}-K$ edges with $K = |\mathcal{E}_t| \times r$, and $r$ is the causal ratio. 


\subsection{Causal-Aware Spatio-Temporal Attention} \label{sec_spatt}
In dynamic graphs, the state of node $u$ at timestamp $t$ is collaboratively determined from its current neighbors and past neighbors, \textit{i.e.}, dynamic ego-graph. Therefore, inspired by some literature \cite{sankar2020dysat,zhang2022dynamic}, we propose causal-aware spatio-temporal attention to learn the invariant spatio-temporal pattern of each node with its causal dynamic ego-graph, further to project the intrinsic evolution rationale behind the causal dynamic subgraphs into node embedding. 

Specifically, we first construct the causal dynamic ego-graph of node $u$ at timestamp $t$ based on the identified causal dynamic subgraphs, denoted as $\mathcal{G}_{1:t}^{c,u} = \left\{\mathcal{G}_{1}^{c,u}, \mathcal{G}_2^{c,u},\cdots, \mathcal{G}_t^{c,u}\right\}$. For node $u$ at timestamp $t' (t' \leq t)$, we utilize spatial self-attention integrated with spatial-temporal scores $\alpha_{t'}^{u,v}$ to aggregate information from the neighbors of $\mathcal{G}_{t'}^{c,u}$. On the one hand, $\alpha_{t'}^{u,v}$ helps learn causal information with diversity, thereby aiding the model in effectively extracting invariant patterns. On the other hand, involving the dynamic causal subgraph generator in the optimization process allows for end-to-end training. This process is formalized as follows,

\begin{equation}
    \mathbf{z}_{t'}^{u} = \sum_{v \in \mathcal{N}_{t'}^{u}} \beta_{t'}^{u,v} \mathbf{v}_{t'}^{v},
     \beta_{t'}^{u,v} = \frac{exp \left(\frac{\mathbf{q}_{t'}^{u} (\mathbf{k}_{t'}^{v})^{T}}{\sqrt{d'}} \alpha_{t'}^{u,v} \right) }{\sum_{j\in \mathcal{N}_{t'}^{u}} exp\left(\frac{\mathbf{q}_{t'}^{u} (\mathbf{k}_{t'}^{j})^{T}}{\sqrt{d'}} \alpha_{t'}^{u,j} \right) }
\end{equation}
\begin{equation}
    \mathbf{q}_{t'}^{u} = \mathbf{W}^{q} \mathbf{X}_{t'}^{u}, \hspace{1em} \mathbf{k}_{t'}^{v} = \mathbf{W}^{k} \mathbf{X}_{t'}^{v}, \hspace{1em} \mathbf{v}_{t'}^{v} = \mathbf{W}^{v} \mathbf{X}_{t'}^{v},
\end{equation}
where $\mathbf{z}_{t'}^{u}$ summarizes the causal patterns at timestamp $t'$ ,$\mathcal{N}_{t'}^{u} = \left\{ v : (u,v) \in \mathcal{E}_{t'}^{c}\right\}$ is the \textit{L}-hop causal neighbors of node \textit{u} at timestamp \textit{t'}, and $d'$ denotes the feature dimension. $\mathbf{q}_{t'}^{u}$, $\mathbf{k}_{t'}^{v}$, $\mathbf{v}_{t'}^{v}$ represents the query vector of node \textit{u}, key and value vector of node \textit{v} in timestamp \textit{t'}, respectively. 

After getting causal patterns $\mathbf{z}_{1:t'}^{u}$ of node $u$ across various snapshots, we employ a temporal self-attention mechanism to aggregate information from its causal dynamic ego-graph $\mathcal{G}_{1:t}^{c,u}$ to obtain the final invariant spatio-temporal node embedding $\mathbf{z}_t^u$ of node $u$ at timestamp $t$. Specially, 
\begin{equation}
     \mathbf{z}_{t}^{u} = \sum_{t' = 1}^{t} \gamma_{t,t'}^{u} \mathbf{v}_{t'}^{u}, \hspace{1em}
     \gamma_{t,t'}^{u} = \frac{exp \left(\frac{\mathbf{q}_{t}^{u} (\mathbf{k}_{t'}^{u})^{T}}{\sqrt{d'}} \right)}{\sum_{i = 1}^{t} exp\left(\frac{\mathbf{q}_{t}^{u} (\mathbf{k}_{i}^{u})^{T}}{\sqrt{d'}}\right)}
\end{equation}

\begin{equation}
\begin{aligned}
    \mathbf{q}_t^{u} = \mathbf{W}^{q} & \left(\mathbf{z}_{t}^{u} || TE(t)\right) , \mathbf{k}_{t'}^{u} = \mathbf{W}^{k} \left(\mathbf{z}_{t'}^{u} || TE(t')\right) , \\ 
    &\hspace{2em}
    \mathbf{v}_{t'}^{u} = \mathbf{W}^{v} \left(\mathbf{z}_{t'}^{u} || TE(t')\right)
\end{aligned}
\end{equation}
where $\mathbf{q}_{t}^{u}$, $\mathbf{k}_{t'}^{u}$, $\mathbf{v}_{t'}^{u}$ represents the query, key and value vector of node \textit{u} in timestamp \textit{t} and \textit{t'}, respectively. $TE(t)$ is utilized to account for the inherent temporal dependencies across different snapshots.

\subsection{Adaptive Environment Generator}
Inferring environment instances directly from observed data suffers from two key drawbacks: it yields only a limited number of discrete environments and fails to reflect the continuous nature of dynamic OOD shifts. For example, in citation networks, authors gradually transition between research fields. To overcome this, we infer the environment distribution and sample a large number of semantically similar yet distinct instances. These continuous environment variations help the model capture the underlying dynamics of distributional shifts, thereby improving its generalization ability across complex OOD scenarios. Specifically,
\begin{equation} 
q(\mathbf{E}_{t}|\mathbf{X}_{t}^{e}, \mathbf{A}_t^{e}) = \prod_{i=1}^{n} q(\mathbf{E}_{t}^{i}| \mathbf{X}_t^e, \mathbf{A}_t^e),
\label{eq12}
\end{equation} 
\begin{equation} 
q(\mathbf{E}_{t}^{i}| \mathbf{X}_t^e, \mathbf{A}_t^t) = \mathcal{N}(\mathbf{E}_{t}^{i}| \boldsymbol{\mu}_{t}^{i}, diag((\boldsymbol{\sigma}^{i}_{t})^{2})),
\label{eq13}
\end{equation}
where $\mathbf{X}_t^e$ and $\mathbf{A}_t^e$ denote the feature matrix and adjacency matrix of variant subgraph $\mathcal{G}_t^e$, respectively, $\mathbf{E}_t^i$ represents the \textit{i}-th vector of $\mathbf{E}_t$, $\boldsymbol{\mu}_{t}^{i}$ and $(\boldsymbol{\sigma}^{i}_{t})^{2}$ are the \textit{i}-th values of environment mean vectors $\boldsymbol{\mu}_{t}$ and environment variance vector $(\boldsymbol{\sigma}_{t})^{2}$. Then we sample sufficient environment instances $\mathbf{e}_{i}$ from the approximate posterior distribution $q(\mathbf{E}_{t}|\mathbf{X}_{t}^{e}, \mathbf{A}_t^{e})$ of latent environment to conduct causal intervention, which can endow the model with stronger generalization capabilities across various complex dynamic OOD scenarios.

Additionally, to accurately characterize the environmental distribution, we take a new prior distribution on the environment by allowing for prior parameters to be modeled by information of variant dynamic subgraphs. In particular, we can write the construction of the prior distribution adopted in our work as follows,
\begin{equation} \label{eq14}
    p(\mathbf{E}_t|\mathbf{X}_t^e, \mathbf{A}_t^e, t) = \mathcal{N}(\boldsymbol{\mu}_{t}^{p}, diag((\boldsymbol{\sigma}^{p}_{t})^{2})),   
\end{equation}
\begin{equation} \label{eq15}
(\boldsymbol{\mu}_{t}^{p},\boldsymbol{\sigma}^{p}_{t}) = {\rm MLP}({\rm GNN}(\mathbf{X}_t^e, \mathbf{A}_t^e)|| {TE(t)}),
\end{equation}
where $\boldsymbol{\mu}_{t}^{p}$ and $\boldsymbol{(\sigma}^{p}_{t})^{2})$ are the learned prior mean vector and variance vector, respectively.  Here, the temporal encoding function is employed to preserve the temporal dependency information in prior distribution. We can enforce the approximate posterior $q(\mathbf{E}_{t}|\mathbf{X}_{t}^{e}, \mathbf{A}_t^{e})$ to be close to the prior $ p(\mathbf{E}_t|\mathbf{X}_t^e, \mathbf{A}_t^e, t)$ by minimizing their KL divergence to optimize the process,
\begin{equation} \label{eq_kl_loss}
    \mathcal{L}_{E} = {\rm KL}[q(\mathbf{E}_{t}|\mathbf{X}_{t}^{e}, \mathbf{A}_t^{e})||p(\mathbf{E}_t|\mathbf{X}_t^e, \mathbf{A}_t^e, t)].
\end{equation} 

\subsection{Optimization with Causal Intervention} 
Based on the learned invariant spatio-temporal embedding $\mathbf{z}$ and generated environment instance $\mathbf{e}_i$, we further instantiate the training objective of OOD generalization with causal intervention. We first utilize two classifiers to project $\mathbf{z}$ and $\mathbf{e}_i$ into a probability distribution over labels $\mathbf{y}$, where $\mathbf{z}$ and $\mathbf{y}$ are the summarized invariant node embeddings and labels. Inspired by \cite{wu2022dir,cadene2019rubi}, then we formulate the joint prediction $\hat{\mathbf{y}}_{do(E = e_i)}$ under the intervention $do(\mathbf{E} = \mathbf{e}_i)$ as $\hat{\mathbf{y}}_{c}$ masked by $\hat{\mathbf{y}}_{e_i}$,
\begin{equation} \label{eq_inervention}
    \hat{\mathbf{y}}_{do(\mathbf{E} = \mathbf{e}_i)} = \hat{\mathbf{y}}_{c} \odot \sigma(\hat{\mathbf{y}}_{e_i}),\hspace{1em}
    \hat{\mathbf{y}}_{c} = \Phi_c(\mathbf{z}), \hspace{1em}
    \hat{\mathbf{y}}_{e_i} = \Phi_e(\mathbf{e_i}), 
\end{equation}
where $\hat{\mathbf{y}}_{c}$ is the predicted label obtained solely from the causal dynamic subgraphs, $\hat{\mathbf{y}}_{e_i}$ measures the predictive ability of the variant dynamic subgraphs form environment, and $\odot$ denotes element-wise multiplication, $\sigma$ is the sigmoid function used to help discover the causal subgraph via adjusting the output logits of $\mathbf{z}$. Then we calculate the the intervention loss and task loss as follows, 
\begin{equation} \label{eq_causal+loss}
    \mathcal{L}_{do} = {\rm Var}_{\mathbf{e}_i \in \mathbf{E}} \left\{ \ell (\hat{\mathbf{y}}_{do(\mathbf{E} = \mathbf{e}_i)}, y)\right\},\hspace{1em}
    \mathcal{L}_{inv} = \ell (\hat{\mathbf{y}}_{c}, y).
\end{equation} 

The final loss function consists of three parts: invariance loss, intervention loss, and the KL divergence of the environmental distribution,
\begin{equation} \label{eq_loss}
    \mathcal{L} = \mathcal{L}_{inv} + \lambda (\mathcal{L}_{do} + \mathcal{L}_{E}),
\end{equation}
where $\lambda$ is the trade-off weight. We provide the overall training algorithm, complexity analysis, and time cost of DyCIL in Appendix \ref{app_alp_noration}, \ref{app_complex}, and \ref{app_experiments}.

\begin{table*}[ht]
\centering
\resizebox{0.9\textwidth}{!}{
\begin{tabular}{ccccccccc}
\toprule
Dataset          & \# Nodes & \# Edges & \# Snapshots & \# Features & \# Classes & Train/Val/Test & \begin{tabular}[c]{@{}c@{}}Temporal\\ Granularity\end{tabular} & Distribution Shift        \\ \midrule
Collab           & 23,035   & 151,790  & 16           & 32          & -          & 10/1/5         & year                 & Cross-Domain Transfers    \\ 
ACT              & 20,408   & 202,339  & 30           & 32          & -          & 20/2/8         & day                  & Cross-Domain Transfers    \\ 
Synthetic-Collab & 23,035   & 151,790  & 16           & 64          & -          & 10/1/5         & -                    & Feature Evolution         \\ 
Temporal-Motif   & 5,000     & 298,709  & 30           & 8           & 3          & 24/3/3         & -                    & Motif Structure Transfers \\
Aminer           & 43,141   & 851,527  & 17           & 128         & 20         & 11/3/3         & year                 & Temporal Evolution        \\ \bottomrule
\end{tabular}}
\caption{The summary of the dataset statistics.}
\label{tab_datasets}
\end{table*}

\begin{table*}[t]
\centering
\resizebox{0.9\linewidth}{!}{
\begin{tabular}{c|cccccccc}
\toprule
\multicolumn{1}{l|}{}                           & \multicolumn{1}{c|}{Dataset}                       & \multicolumn{2}{c}{Collab}                                     & \multicolumn{2}{c}{ACT}                                        & \multicolumn{3}{c}{Synthetic-Collab}                            \\ \midrule
                          \multicolumn{1}{c|}{Method Type} & \multicolumn{1}{c|}{Model}   & w/o OOD           & \multicolumn{1}{c|}{w/ OOD}            & w/o OOD           & \multicolumn{1}{c|}{w/ OOD}            & $\bar{p} = 0.4$                 & $\bar{p} = 0.6$                  & $\bar{p} = 0.8$                \\ \midrule
                                                \multirow{3}{*}{DyGNNs} & \multicolumn{1}{c|}{GCRN}     & 82.78±0.54          & \multicolumn{1}{c|}{69.72±0.45}          & 76.28±0.51          & \multicolumn{1}{c|}{64.35±1.24}          & 72.57±0.72          & 72.29±0.47          & 67.26±0.22          \\
                                                & \multicolumn{1}{c|}{EvolveGCN}     & 86.62±0.95          & \multicolumn{1}{c|}{76.15±0.91}          & 74.55±0.33          & \multicolumn{1}{c|}{63.17±1.05}          & 69.00±0.53          & 62.70±1.14          & 60.13±0.89          \\
                                                & \multicolumn{1}{c|}{DySAT}    & 88.77±0.23          & \multicolumn{1}{c|}{76.59±0.20}          & 78.52±0.40          & \multicolumn{1}{c|}{66.55±1.21}          & 70.24±1.26          & 64.01±0.19          & 62.19±0.39          \\ \midrule
\multirow{3}{*}{OOD}                            & \multicolumn{1}{c|}{IRM}      & 87.96±0.90          & \multicolumn{1}{c|}{75.42±0.87}          & 80.02±0.57          & \multicolumn{1}{c|}{69.19±1.35}          & 69.40±0.09          & 63.97±0.37          & 62.66±0.33          \\
                                                & \multicolumn{1}{c|}{VREx}     & 88.31±0.32          & \multicolumn{1}{c|}{76.24±0.77}          & 83.11±0.29          & \multicolumn{1}{c|}{70.15±1.09}          & 70.44±1.08          & 63.99±0.21          & 62.21±0.40          \\
                                                & \multicolumn{1}{c|}{GroupDRO} & 88.76±0.12          & \multicolumn{1}{c|}{76.33±0.29}          & 85.19±0.53          & \multicolumn{1}{c|}{74.35±1.62}          & 70.30±1.23          & 64.05±0.21          & 62.13±0.35          \\ \midrule
\multicolumn{1}{c|}{\multirow{2}{*}{Graph OOD}} & \multicolumn{1}{c|}{EERM}     & OOM                 & \multicolumn{1}{c|}{OOM}                 & OOM                 & \multicolumn{1}{c|}{OOM}                 & OOM                 & OOM                 & OOM                 \\
\multicolumn{1}{l|}{}                           & \multicolumn{1}{c|}{DIR}      & 88.18±0.12          & \multicolumn{1}{c|}{76.07±0.30}          & 88.24±0.46          & \multicolumn{1}{c|}{75.71±2.15}          & 76.81±1.11          & 67.97±0.35          & 66.53±0.60           \\  \midrule
\multirow{5}{*}{\begin{tabular}[c]{@{}c@{}}Dynamic Graph\\ OOD\end{tabular}}                      & \multicolumn{1}{c|}{DIDA}     & 91.97±0.05          & \multicolumn{1}{c|}{81.87±0.40}          & 89.84±0.82          & \multicolumn{1}{c|}{78.64±0.97}          & 85.20±0.84          & 82.89±0.23          & 72.59±3.31          \\
                                                & \multicolumn{1}{c|}{SILD}     & \underline{92.36±0.19}          & \multicolumn{1}{c|}{84.14±0.31}          & 89.28±0.41          & \multicolumn{1}{c|}{79.91±0.65}          & 85.95±0.18          & 84.69±1.18          & 78.01±0.71          \\
                                                & \multicolumn{1}{c|}{EAGLE}    & 92.34±0.13          & \multicolumn{1}{c|}{84.24±0.19}          & \underline{92.43±0.71}          & \multicolumn{1}{c|}{82.99±0.83}          & \underline{88.12±0.47}          & \underline{86.97±0.50}          & \underline{82.08±0.96}
                                                \\
                                                & \multicolumn{1}{c|}{OOD-Linker}    & ---          & \multicolumn{1}{c|}{\underline{85.30±0.31}}          & ---          & \multicolumn{1}{c|}{\textbf{85.98±1.00}}          & 85.58±1.54          & 83.09±1.82          & 79.83±1.69
                                                \\ \cmidrule{2-9}
                                                & \multicolumn{1}{c|}{DyCIL}    & \textbf{95.00±0.09} & \multicolumn{1}{c|}{\textbf{86.54±0.17}} & \textbf{93.41±0.13} & \multicolumn{1}{c|}{\underline{85.66±0.21}} & \textbf{90.98±0.09} & \textbf{87.69±0.23} & \textbf{83.31±0.19} \\ \bottomrule
    \end{tabular}}
\caption{AUC score (\%) of link prediction. The best are bolded and the second best are underlined.}
\label{tab_lp_results}
\end{table*}

\begin{table*}[t]
\centering
\resizebox{0.8\textwidth}{!}{
\begin{tabular}{c|c|ccc|ccc}
\toprule
\multicolumn{1}{l|}{}      & Dataset  & \multicolumn{3}{c}{Temporal-Motif}                             & \multicolumn{3}{c}{Aminer}                                      \\ \midrule
Method Type                & Model   & T-M28               & T-M29               & T-M30               & Aminer15            & Aminer16            & Aminer17            \\ \midrule
\multirow{3}{*}{DyGNNs}     & GCRN     & 51.93±1.68          & 51.55±1.26          & 51.31±1.28          & 47.96±1.12          & 51.33±0.62          & 42.93±0.71          \\
                           & EvolveGCN     & 45.63±0.33          & 45.50±0.28          & 45.29±0.42          & 44.14±1.12          & 46.28±1.84          & 37.71±1.84          \\
                           & DySAT    & 48.44±0.56          & 48.20±0.45          & 47.99±0.45          & 48.41±0.81          & 49.76±0.96          & 42.39±0.62          \\ \midrule
\multirow{3}{*}{OOD}       & IRM      & 50.17±0.71          & 50.01±0.52          & 49.84±0.67          & 48.44±0.13          & 50.18±0.73          & 42.40±0.27          \\
                           & VREx     & 51.32±0.77          & 52.64±0.63          & 50.79±0.46          & 48.70±0.73          & 49.24±0.27          & 42.59±0.37          \\
                           & GroupDRO & 50.89±0.38          & 50.95±0.87          & 50.18±0.53          & 48.73±0.61          & 49.74±0.26          & 42.80±0.36          \\ \midrule
\multirow{2}{*}{Graph OOD} & EERM     & 45.23±0.47          & 45.05±0.40          & 44.68±0.41          & OOM                 & OOM                 & OOM                 \\
                           & DIR      & 47.07±0.45          & 47.42±0.68          & 45.74±0.56          & 46.85±1.84          & 49.69±1.71          & 41.73±1.41          \\  \midrule
\multirow{4}{*}{\begin{tabular}[c]{@{}c@{}}Dynamic Graph\\ OOD\end{tabular}} & DIDA     & 64.26±0.31          & 63.71±0.41          & 61.94±0.33          & 50.34±0.81          & 51.43±0.27          & 44.69±0.06          \\
                           & SILD     & \underline{77.89±0.74}          & \underline{72.70±0.84}          & \underline{66.72±0.93}          & \underline{51.61±1.12}          & \underline{52.85±1.05}          & \underline{44.47±1.18}          \\
                           & EAGLE    & 67.26±1.17          & 66.98±1.06          & 65.46±1.05          & OOM                 & OOM                 & OOM                 \\ \cmidrule{2-8} 
                           & DyCIL    & \textbf{82.66±0.34} & \textbf{81.13±0.20} & \textbf{78.95±0.19} & \textbf{53.04±0.75} & \textbf{54.30±0.21} & \textbf{45.89±0.45} \\ \bottomrule
\end{tabular}}
\caption{ACC score (\%) of node classification. The best are bolded and the second best are underlined.}
\label{tab_nc_results}
\end{table*}

\section{Experiments} 
\subsection{Datasets and Baselines}

We utilize three real-world datasets, namely Collab \cite{tang2012cross}, ACT \cite{kumar2019predictingact}, and Aminer \cite{tang2008arnetminer}. Additionally, we synthesize two datasets by introducing manually designed distribution shifts, Temporal-Motif and Synthetic-Collab. Table \ref{tab_datasets} summarizes the statistics of all datasets. 
We compare our method with representative approaches from DyGNNs (GCRN \cite{seo2018structured}, EvolveGCN \cite{pareja2020evolvegcn}, DySAT \cite{sankar2020dysat}), OOD generalization (IRM \cite{arjovsky2019invariant}, V-REx \cite{krueger2021VREx}, GraphDRO \cite{sagawa2019distributionally}), graph OOD generalization (DIR \cite{wu2022dir}, EERM \cite{wu2022handling}, and dynamic graph OOD generalization (DIDA \cite{zhang2022dynamic}, SILD \cite{zhang2023spectral}, EAGLE \cite{yuan2023eagle}, OOD-Linker \cite{tieu2025oodlinker} ). 
More details about the datasets, baselines, reproducibility, and experiment results are in Appendix \ref{app_datasets}, \ref{app_details}, and \ref{app_experiments}.

\subsection{Results on Link Prediction Task}\label{sec_LP}
We show the results in Table \ref{tab_lp_results}. DyGNNs methods deteriorate significantly under distribution shifts, while DyCIL achieves more progress in OOD scenarios.  Specifically, compared to DySAT, DyCIL achieves improvements of 7.01\% and 12.9\%, and 18.9\% and 28.7\% for the w/o OOD and w/ OOD scenarios on the Collab and ACT datasets, respectively. The general OOD generalization methods only have limited improvements, as the environment labels they rely on are not available in real-world dynamic graphs. 
Graph OOD methods focus only on the causal rationale of the structure without considering the temporal dynamics of dynamic graphs, hindering their ability to handle OOD shifts in dynamic graphs. DIDA and SILD rely solely on observable data to perform causal interventions, which weakens their generalization ability in scenarios with severe distribution shifts. Specifically, in the Synthetic-Collab dataset with $\bar{p}=0.8$, the performance of DIDA and SILD decrease by 10.3\% and 6.68\% compared to $\bar{p}=0.6$, while our model only drops by 4.41\%. This demonstrates that DyCIL has better generalization capabilities in dynamic graphs with severe OOD shifts by capturing the underlying dynamics of distributional shifts adaptively. EAGLE achieves further improvements by focusing on latent environment modeling of dynamic graphs. However, it requires additional labels to infer the environment and overlooks the intrinsic evolution rationale of invariant patterns that are commonly present in real dynamic graphs, which hinders its scalability and performance. Therefore, DyCIL shows a more significant improvement in the Collab and ACT real-world datasets compared to EAGLE. Due to its neglect of environmental influences, OOD-Linker performs significantly worse than EAGLE and DyCIL on three synthetic datasets characterized by severe OOD distribution shifts.

\subsection{Results on Node Classification Task}
The results are shown in Table \ref{tab_nc_results}. DyCIL outperforms all baseline methods on both node classification datasets by a significantly large margin. Specifically, DyCIL improves by 4.77\%, 8.43\%, and 12.23\% over the best baseline, SILD, on the Temporal-Motif dataset.  As time evolves, the improvement of DyCIL becomes more significant, mainly because the emergence of more unseen variant motif structures leads to more severe distribution shifts as time goes on. SILD struggles to handle cases with severe OOD shifts, resulting in a noticeable performance decline over time. In contrast, DyCIL can continuously capture causal motif structures under different degrees of distribution shifts, thus significantly outperforming various baselines. On the Aminer dataset, DyCIL handles OOD shifts better than all the baselines, which validates that DyCIL can capture invariant spatio-temporal patterns under distribution shifts. In addition, DyCIL exhibits smaller variance in most cases, showing it's less sensitive to spurious correlations under different distribution shifts.

\begin{figure}[ht]
\centering
\subfloat{
\centering
\includegraphics[width=0.43\linewidth]{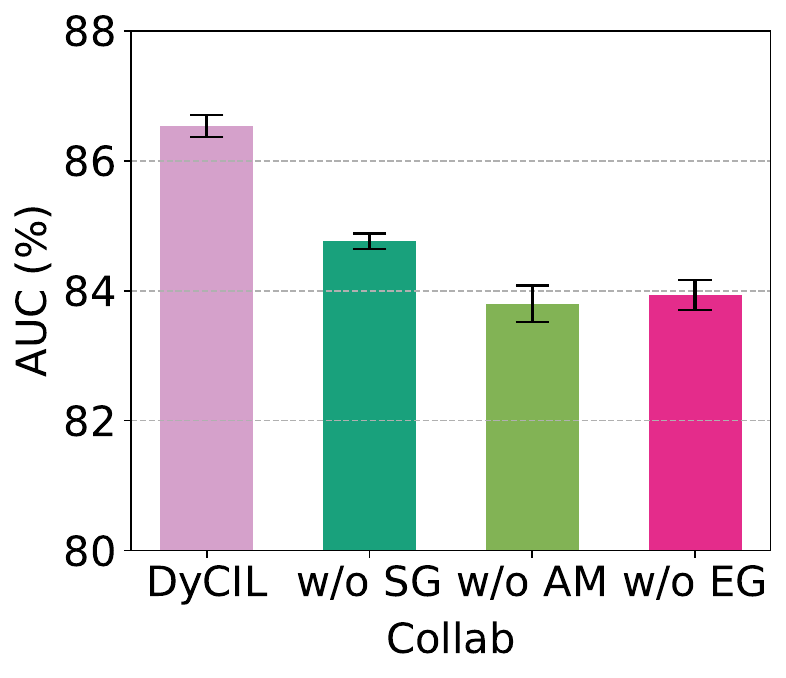}
}%
\subfloat{
\centering
\includegraphics[width=0.43\linewidth]{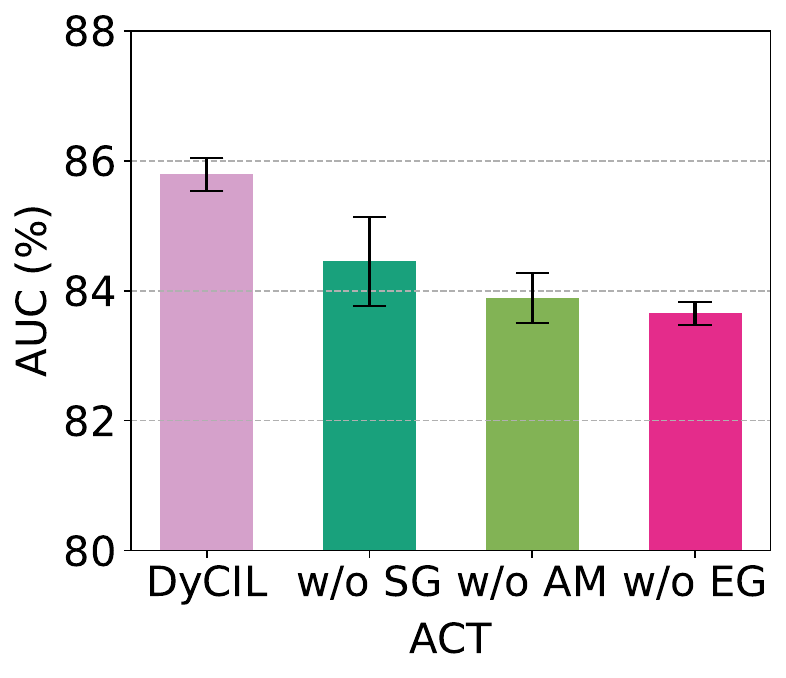}
}%
\caption{The ablation experiment results.}
\label{figure_ablation}
\end{figure}
\subsection{Ablation Study}
To further validate the contribution of each component in our model, we conduct ablation experiments. We name the DyCIL variants as follows, \textbf{w/o SG}: DyCIL without the dynamic causal subgraph generator. \textbf{w/o AM}: DyCIL without the causal-aware spatio-temporal attention module. \textbf{w/o EG}: DyCIL without adaptive environment generator. We present the ablation experiment results in Figure \ref{figure_ablation}. Overall, DyCIL outperforms all its variants across various datasets. Removing any of the three components leads to a significant performance drop, further validating the effectiveness of each component of our model. In summary, DyCIL jointly optimizes three mutually promoting modules to effectively capture invariant spatio-temporal patterns and achieve OOD generalization under distribution shifts in dynamic graphs. More ablation analyses are provided in Appendix \ref{app_experiments}.

\begin{figure}[t]
\centering
\includegraphics[width=0.65\linewidth]{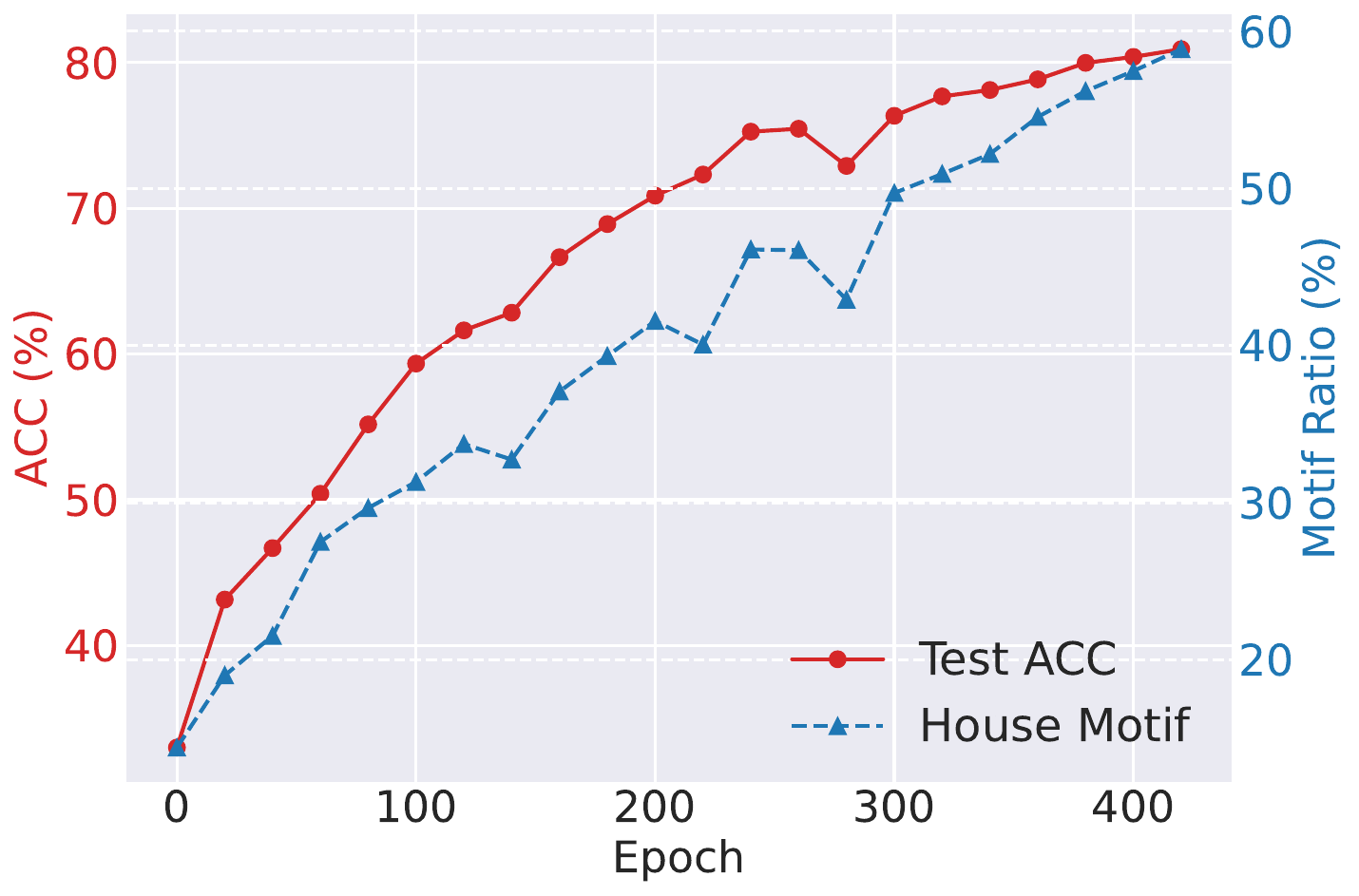} %
\caption{A case study on Temporal-Motif dataset. }
\label{fig_case}    
\end{figure}

\subsection{Case Study} \label{}
To verify whether DyCIL truly learns the causal dynamic subgraphs and captures the evolutionary rationale in dynamic graphs, we conduct a case study on the Temporal-Motif dataset. Since the node labels are determined solely by the evolving 'house' motif that constitutes the causal rationale, the causal dynamic subgraphs generated by the model should contain more 'house' motifs. Therefore, we calculate the ratio of 'house' motifs in the causal dynamic subgraphs to the total number of 'house' motifs and show its relationship with performance in Figure \ref{fig_case}. We can notice that as the number of training epochs increases, the identified causal dynamic subgraphs contain more 'house' motifs, enabling successive improvements in performance. In other words, the performance improvement of DyCIL comes from capturing more 'house' motifs. It demonstrates that DyCIL captures genuine causal patterns to make stable predictions across various distribution shifts and explains why DyCIL has good OOD generalization ability. This serves as additional evidence that DyCIL enhances interpretability by identifying causal subgraphs in dynamic graph OOD generalization. 


\section{Conclusion}
In this paper, we propose DyCIL, a novel model to handle OOD shifts in dynamic graphs by exploiting spatio-temporal invariant patterns from a causal view. DyCIL first generates causal dynamic subgraphs via a dynamic causal subgraph generator. Then DyCIL learns invariant node representations by capturing the evolutionary rationale behind its causal dynamic ego-graph. Finally, DyCIL instantiates latent environments through inferred environment distribution and performs causal interventions. Experiment results demonstrate that DyCIL can better handle OOD shifts compared to state-of-the-art baselines. The ablation study further validates the contribution of each module. Additionally, we conduct a case study to demonstrate DyCIL's capability to identify causal dynamic subgraphs. One limitation is that we mainly consider discrete dynamic homogeneous graphs. Future work can explore OOD generalization in dynamic heterogeneous graphs and continuous dynamic graphs, which are more common and challenging in the real world.


\section{Acknowledgments}
This work was supported in part by the Zhejiang Provincial Natural Science Foundation of China under Grant LDT23F01015F01 and LMS25F030011, in part by the Key Technology Research and Development Program of the Zhejiang Province under Grant No. 2025C1023, in part by the National Natural Science Foundation of China under Grants 62372146, and in part by the Open Research Fund of The State Key Laboratory of AI Safety, Chinese Academy of Sciences (2025-02).

\bibliography{aaai2026}

\appendix
\setcounter{secnumdepth}{2}
\section{Notations and Algorithm}  \label{app_alp_noration}
We summarize the symbols used in this paper in Table \ref{notation} and provide the overall training algorithm for DyCIL in Algorithm \ref{alg_dycil}. In real-world dynamic graphs, distribution shifts are common due to the confounding environmental biases in their generation process. For example, in a dynamic social network, individuals may change jobs or living places, and these changes in their living area act as confounding biases from the external environment, implicitly affecting their attributes and behaviors.

\begin{table}[ht]
\caption{Major notations in this paper}
\label{notation}
\renewcommand{\arraystretch}{1}
\centering
  \begin{tabular}{c|l}
    \toprule
    Symbol & Description \\
    \midrule
    $\mathcal{G}_{1:T}$ & A dynamic graph with $T$ snapshots \\
    $\mathcal{G}_t$ & The snapshot at timestamp $t$ \\
    $\mathcal{V}_t, \mathcal{E}_t$ & The nodes and edges of the snapshot $\mathcal{G}_t$ \\
    $Y_t$, $Y_t^{u}$ & Labels of $\mathcal{G}_t$ and the label of node $u$ in $\mathcal{G}_t$ \\
    $\mathbf{X}_t, \mathbf{A}_t$ & The feature matrix and adjacency matrix of $\mathcal{G}_t$ \\ 
    $\mathcal{G}_{1:T}^{c}$ & The causal dynamic subgraph of $\mathcal{G}_{1:T}$ \\
    $\mathcal{G}_t^{c}$ & The causal subgraph at timestamp $t$ \\
    $\mathcal{N}_{t}^{u}$ & The \textit{L}-neighbors of node $u$ in $\mathcal{G}_t^{c}$  \\
    $\mathcal{G}_{1:T}^{e}$ & The variant dynamic subgraph of $\mathcal{G}_{1:T}$ \\
    $\mathcal{G}_{t}^{e}$ & The variant subgraph of $\mathcal{G}_{1:T}$ at timestamp $t$ \\
    $\mathcal{E}_t^{c}, \mathcal{E}_t^{e}$ & The edges of  $\mathcal{G}_t^{c}$ and  $\mathcal{G}_{t}^{e}$ \\
    $\mathbf{X}_t^e, \mathbf{A}_t^e$ & The feature matrix and adjacency matrix of $\mathcal{G}_t^{e}$ \\ 
    \midrule
    $||$ & The concatenation operation \\
    $D_t^u$ & The degree of node $u$ in $\mathcal{G}_{t}$\\
    $\alpha_{t}^{u,v}$ & The spatio-temporal score of edge $(u,v)$ at $\mathcal{G}_{t'}$ \\
    $\mathbf{h}_t$ & The spatio-temporal matrix of $\mathcal{G}_{t}$ \\
    $h_t^u, h_t^v$ & The spatio-temporal vector of node $u$ and $v$ in $\mathcal{G}_{t}$ \\
    $\beta_{t'}^{u,v}$ & The attention weight of edge $(u,v)$ at $\mathcal{G}_{t'}$ \\
    $\gamma_{t,t'}^{u}$ & The weight between node $u$  at $\mathcal{G}_{t'}$ and $\mathcal{G}_{t}$ \\
    $\mathbf{q}, \mathbf{k}, \mathbf{v}$ & The query, key and value vector \\
    $\boldsymbol{\mu}_{t}, (\boldsymbol{\sigma}_{t})^{2}$ & The mean and variance of environment posterior \\
    $\boldsymbol{\mu}_{t}^{p}, (\boldsymbol{\sigma}_{t}^{p})^{2}$ & The mean and variance of environment prior \\
    $\mathbf{e}_i$ & The sampled environment instance \\
    $\Phi_c, \Phi_e$ & The classifier for causal and variant prediction \\
    $\mathbf{z}$ & The summarized invariant node embeddings \\
    $\mathbf{y}$ & The summarized labels \\
    
    \bottomrule 
\end{tabular} 
\end{table}

\begin{algorithm}[t]
    \caption{Whole Training Algorithm of DyCIL}
    \label{train}
    \renewcommand{\algorithmicrequire}{\textbf{Input:}}
    \renewcommand{\algorithmicensure}{\textbf{Output:}}
    \begin{algorithmic}[1]
        \REQUIRE The dynamic graph $\mathcal{G}_{1:T}=\left\{\mathcal{G}_1, \mathcal{G}_2,\cdots, \mathcal{G}_T\right\}$ with labels $\mathbf{y}=\left\{{Y}_1, {Y}_2,\cdots, {Y}_T\right\}$; Number of training epochs $L$; Number of causal interventions $n$; Causal ratio $r$ and  trade-off wight $\lambda$
        \ENSURE Optimized model $f_{\theta^{}}$ with parameters $\theta$   
        \FOR{ $j=1,2, \cdots, $ To $L$}
            
            \STATE Calculate the spatio-temporal scores $\alpha$ of edges by Eq. \ref{eq_edge_score}
            \STATE Obtain the causal dynamic subgraphs $\mathcal{G}^{c}_{1:T}$ by Eq. \ref{eq_subgraph}. and collect complement of $\mathcal{G}^{c}_{1:T}$ as the variant dynamic subgraph $\mathcal{G}^e_{1;T}$ 
            \STATE Get the invariant spatio-temporal node embedding $\mathbf{z}$ as described in Section \ref{sec_spatt}
            \STATE Infer the posterior distribution $q(\mathbf{E}_{t}|\mathbf{X}_{t}^{e}, \mathbf{A}_t^{e})$ of environment by Eq. \ref{eq12}
            \STATE Learn the prior distribution $p(\mathbf{E}_t|\mathbf{X}_t^e, \mathbf{A}_t^e, t)$ of environment by Eq. \ref{eq14}
            \STATE Calculate the KL divergence $\mathcal{L}_{E}$ of the environment distribution by Eq. \ref{eq_kl_loss}
            \FOR{$i = 1,2, \cdots, n $}
 
                \STATE Sample environment instances $\mathbf{e}_i$ from $q(\mathbf{E}_{t}|\mathbf{X}_{t}^{e}, \mathbf{A}_t^{e})$

                \STATE Conduct causal intervention by Eq. \ref{eq_inervention} 
            \ENDFOR
            \STATE Calculate the intervention loss and task loss by Eq. \ref{eq_causal+loss}
            \STATE Calculate the overall loss $\mathcal{L}$ by Eq. \ref{eq_loss}
            \STATE Update model parameters by minimizing $\mathcal{L}$
        \ENDFOR
        \RETURN Optimized model $f_{\theta^{}}$ with parameters $\theta$
    \end{algorithmic} \label{alg_dycil}
\end{algorithm}

\section{Derivation of BackDoor Adjustment}
\label{app_causal}

In this section, we present the derivation of backdoor adjustment based on some rules \cite{pearl2000models,pearl2016causal}.
\begin{itemize}[leftmargin=*]
    \item \textit{Rule 1: Bayes Rule}:
    \begin{equation}
    \nonumber
        P(Y|X) = \sum\nolimits_{Z} P(Y|X,Z)P(Z|X)
    \end{equation}
        \item \textit{Rule 2: Equivalence Rule}:
    \begin{equation}
    \nonumber
        P(Y|X) = P(Y|X_1,X_2), \quad if (Y|X),  X=\left\{X_1, X_2\right\}
    \end{equation}
    \item \textit{Rule 3: Independency Rule}:
    \begin{equation}
    \nonumber
        P(Y|do(X), do(Z)) = P(Y|do(X)), \quad if (Y \bot Z |X)
    \end{equation}
    \item \textit{Rule 4: Exchange Rule}:
    \begin{equation}
    \nonumber
        P(Y|do(X), do(Z)) = P(Y|do(X), Z), \quad if (Y \bot Z |X)
    \end{equation}

\end{itemize}
Based on the above rules, we can obtain the following,
\begin{equation}
\nonumber
\begin{alignedat}{2}
    P({Y}_t|do({C}_t)) &= \sum\nolimits_{e} P(Y_t|do({C}_t),{E}_t=e)P({E}_t=e|do({C}_t)) &&\quad \\
    &= \sum\nolimits_{e}  P(Y_t|do({C}_t),{E}_t=e)P(E_{t}=e) &&\quad \\
    &= \sum\nolimits_{e}  P(Y_{t}|C_t,E_t=e)P(E_t=e) &&\quad  \\
     &= \sum\nolimits_{e}  P(Y_{t}|S_t,T_{t'},E_t=e)P(E_t=e) &&\quad 
\end{alignedat}
\end{equation}

Then we can stratify $E$ into discrete components $E=\left\{{e_i}\right\}_{i=1}^{|E|}$. Then, we can formalize it as follows,
\begin{equation}
    \nonumber
    P({Y}_t|do({C}_t)) = \sum_{e_{i} \in {E}_t} P({Y}_t|{S}_t, {T}_{t'}, {E}_{t}=e_{i}) P({E}_{t} =e_{i}).
\end{equation}

\section{Proof of Theorem}\label{app_theorem}

\begin{theorem}
Let $\Psi$ be a subgraph generator mapping from $\mathcal{G}_{1:t}$ to a subgraph. Under Assumption 1, the following results hold:

\begin{itemize}
    \item[(i)] The optimal causal subgraph generator $\Psi^*$ satisfies:
    \begin{equation} 
        \Psi^* = \arg\max_{\Psi} I(\Psi(\mathcal{G}_{1:t}); Y_t). \nonumber
    \end{equation}
    
    \item[(ii)] For any generator $\Psi$, the mutual information can be lower bounded by a variational approximation:
    \begin{equation}
        I(\Psi(\mathcal{G}_{1:t}); Y_t) \geq \mathbb{E}_{p(\mathcal{G}_{1:t}, Y_t)} \left[ \log q_\phi(Y_t | \Psi(\mathcal{G}_{1:t})) \right] + H(Y_t), \nonumber
    \end{equation}
    where $q_\phi$ is any variational distribution approximating $p(Y_t|\Psi(\mathcal{G}_{1:t}))$, parameterized by $\phi$, and $H(Y_t)$ is the entropy of $Y_t$.
\end{itemize}
\end{theorem}

\textit{
where $ I(\mathcal{G}_{1:t}; Y_{t})$ is the mutual information between the dynamic graph and label.}

\begin{proof}
For \textit{(i)}, let $\hat{\Psi} = \arg\max_{\Psi} I(\Psi(\mathcal{G}_{1:t}); Y_t)$, based on $\forall \mathbf{e}_i, \mathbf{e}_j \in \mathbf{E}$, $p({Y}_{t}|\Psi (\mathcal{G}_{1:t}), \mathbf{e}_i) = p({Y}_{t}|\Psi (\mathcal{G}_{1:t}), \mathbf{e}_j)$, it’s sufficient to show that $I(\Psi^{*}(\mathcal{G}_{1:t}); Y_{t})) \geq I(\hat{\Psi}(\mathcal{G}_{1:t}); Y_{t})$ to to demonstrate that $\hat{\Psi} = {\Psi}^{*}$.

We introduce a random variable $\tilde{\Psi}$ based on the functional representation lemma \cite{el2011network}, such that $\tilde{\Psi}(\mathcal{G}_{1:t}) \bot \Psi^{*}(\mathcal{G}_{1:t})$ and $\hat{\Psi}(\mathcal{G}_{1:t}) = \gamma(\Psi^{*}(\mathcal{G}_{1:t}), \tilde{\Psi}(\mathcal{G}_{1:t}))$, where $\gamma$ is a function. And, we can have that:
\begin{equation}
    I(\hat{\Psi}(\mathcal{G}_{1:t}); Y_{t}) = I(\gamma(\Psi^{*}(\mathcal{G}_{1:t}), \tilde{\Psi}(\mathcal{G}_{1:t})); Y_{t}) \nonumber
\end{equation}
Based on properties of mutual information and independence constraints, the above equation can be decomposed as follows:

\begin{equation}
    \begin{aligned}
        I(\gamma(\Psi^{*}(\mathcal{G}_{1:t})&, \tilde{\Psi}(\mathcal{G}_{1:t})); Y_{t}) \\
        & \leq I(\Psi^{*}(\mathcal{G}_{1:t}), \tilde{\Psi}(\mathcal{G}_{1:t}); Y_{t}) \\ \nonumber
        & = I(\Psi^{*}(\mathcal{G}_{1:t}), \tilde{\Psi}(\mathcal{G}_{1:t}); f(\Psi^{*}(\mathcal{G}_{1:t})) + \epsilon) \\ 
        \nonumber
    \end{aligned}
\end{equation}
Moreover, since $\tilde{\Psi}(\mathcal{G}_{1:t}) \bot \Psi^{*}(\mathcal{G}_{1:t})$, we can derive that:
\begin{equation}
\begin{aligned}
    I(\Psi^{*}(\mathcal{G}_{1:t}), \tilde{\Psi}(\mathcal{G}_{1:t})&; f(\Psi^{*}(\mathcal{G}_{1:t})) + \epsilon) \\
    & =     I(\Psi^{*}(\mathcal{G}_{1:t}); f(\Psi^{*}(\mathcal{G}_{1:t})) + \epsilon) \\ 
    & = I(\Psi^{*}(\mathcal{G}_{1:t}); Y_{t})
    \nonumber
\end{aligned}
\end{equation}
The proof is finished.
\qed

For \textit{(ii)}, We apply the standard variational lower bound of mutual information. For any joint distribution $p(X,Y)$ and any variational distribution $q_\phi(Y|X)$, we have:
\[
I(X; Y) = \mathbb{E}_{p(X,Y)} \left[ \log \frac{p(Y|X)}{p(Y)} \right].
\]
Subtracting and adding $q_\phi(Y|X)$ inside the logarithm yields:
\[
= \mathbb{E}_{p(X,Y)} \left[ \log \frac{q_\phi(Y|X)}{p(Y)} \right] + \mathbb{E}_{p(X,Y)} \left[ \log \frac{p(Y|X)}{q_\phi(Y|X)} \right].
\]
The second term is the KL divergence $\text{KL}(p(Y|X) \| q_\phi(Y|X)) \geq 0$, so we have:
\[
I(X; Y) \geq \mathbb{E}_{p(X,Y)} \left[ \log q_\phi(Y|X) \right] + H(Y),
\]
where $H(Y) = -\mathbb{E}_{p(Y)} \log p(Y)$ is the entropy of $Y$.

Setting $X = \Psi(\mathcal{G}_{1:t})$ gives the desired result:
\[
I(\Psi(\mathcal{G}_{1:t}); Y_t) \geq \mathbb{E}_{p(\mathcal{G}_{1:t}, Y_t)} \left[ \log q_\phi(Y_t|\Psi(\mathcal{G}_{1:t})) \right] + H(Y_t).
\]

The proof is finished.
\end{proof}

\section{Complexity Analysis} \label{app_complex}
We analyze the computational complexity of DyCIL. We define $|\mathcal{N}_t|$ and $|\mathcal{E}_t|$ as the number of nodes and edges in $\mathcal{G}_t$, $|\mathcal{N}_t^c|$ and $|\mathcal{E}_t^c|$ as the number of nodes and edges in $\mathcal{G}_t^c$, and $|\mathcal{N}_t^e|$ and $|\mathcal{E}_t^e|$ as the number of nodes and edges in $\mathcal{G}_t^e$, respectively. For simplicity, we define \(d\) as the unified dimension. For the causal dynamic subgraph generator, the computational complexity is $\mathcal{O}(T(|\mathcal{N}_t|d^2 + |\mathcal{E}_t|d))$, where $T$ is the number of snapshots in $\mathcal{G}_{1:T}$. The causal-aware spatio-temporal attention module has a computational complexity of $\mathcal{O}(T(|\mathcal{N}_t^c|d^2+|\mathcal{E}_t^c|d + |\mathcal{E}_t^c|) + Td^2+T^2|\mathcal{N}_t^c|d)$. For the adaptive environment generator, the computational complexity is $\mathcal{O}(T(|\mathcal{N}_t^e|d^2 + |\mathcal{E}_t^e|d)$. Our causal intervention mechanism has a computational complexity of $\mathcal{O}(T(|\mathcal{E}_t|dn))$, where $n$ is the number of environment instances and is commonly set as a small constant. In summary, the complexity of DyCIL is $\mathcal{O}(T(|\mathcal{N}_t|d^2 + |\mathcal{E}_t|d + |\mathcal{E}_t^c| + d^2 + T|\mathcal{N}_t^c|d + |\mathcal{E}_t|dn )$.  It can be seen that DyCIL has a linear computational complexity with respect to the number of nodes and edges in dynamic graphs, which is similar to most DyGNNs. We conduct time cost experiments in Section \ref{time} to further analyze the efficiency and scalability of DyCIL.

\section{Reproducibility Details}\label{app_details}
\subsection{Details of DyCIL} For the causal subgraph generator, we use a two-layer GCN \cite{kipf2016semi} and a two-layer MLP to calculate the spatio-temporal score for each edge. The layer of causal-aware spatio-temporal attention is set as 2 for the Temporal-Motif dataset and 1 for other datasets. For the adaptive environment generator, we adopt a two-layer GNN to infer the posterior distribution of the environment (Eq. 14 and Eq. 15), and we use a one-layer GCN and a two-layer MLP to learn the prior distribution of the environment ( Eq. 16 and Eq. 17).
\subsection{Hyperparameters}
We adopt the hidden dimension as 64 for Aminer and 32 for other datasets. The number of max epochs is fixed at 1000, and the patience for early stopping is set to 50 for all tasks and datasets based on the the performance of validation set. 
The learning rate is set to 1e-2 and 2e-3 for link prediction and node classification, respectively.
We set the trade-off weight $\lambda$ to 1, 1e-1, 1, 1e-3, 1e-3 for Collab, ACT, Synthetic-Collab, Temporal-Motif, and Aminer respectively. We set causal ratio $r$ to 0.2, 0.2, 0.3, 0.4, 0.1 for Collab, ACT, Synthetic-Collab, Temporal-Motif, and Aminer datasets respectively.  We set the number of causal interventions to $10*T$ for all datasets, \textit{i.e.} the number of environment instances sampled from the distribution $n$ is $10*T$. We only perform causal interventions during the training phase. We use the invariant node embeddings $\mathbf{z}$ to compute the predicted labels in the testing phase.

\subsection{Evaluation Details} For link prediction tasks, we sample negative examples equal in number to the existing edges to compute the binary cross entropy loss $\ell$. We use an inner product as the task classifier and adopt the area under the ROC curve (AUC) scores as the metric. In the Table of experiments, 'w/o OOD' and 'w/ OOD' represent test data without distribution shift and with distribution shift, respectively.  For the node classification tasks, we use a one-layer MLP as the task classifier, employ cross entropy as the loss function $\ell$, and adopt Accuracy (ACC) as the metric. Reported experimental results are the mean and standard deviation of 3 runs to avoid random errors. In the table In the Table of experiments, 'OOM' denotes out of memory in an NVIDIA GeForce RTX 3090 with 24 GB of memory. Since general OOD generalization methods (IRM, VREx, GroupDRO) are not specifically designed for dynamic graphs, we adopt DySAT as their backbone for fair comparison. For all baselines, we follow the same settings described in their original papers. To evaluate Graph OOD methods designed for static graphs  (EERM, DIR, GALA), we transform the dynamic graph into a large static graph and assess their OOD generalization performance in this setting. Since EpoD \cite{yang2024improving} has not released its source code, we don't include it in our comparison.

\subsection{Configurations}
We conduct all experiments under the following environment:
\begin{itemize}[leftmargin=*]
    \item  Operating System: Ubuntu 22.04.1 LTS
    \item  CPU: Intel(R) Xeon(R) Gold 6330 CPU @ 2.00GHz
    \item GPU:  NVIDIA GeForce RTX 3090 with 24 GB of memory
    \item Software: Python: 3.9, Cuda: 11.3, Cudnn:8.3.2\_0 , PyTorch: 1.12.0 \cite{paszke2019pytorch}, PyTorch Geometric: 2.3.0 \cite{fey2019fast}
\end{itemize}

\section{Datasets and Baselines Details} \label{app_datasets}

\subsection{Baselines}
Considering that our research focuses on OOD generalization in dynamic graphs, we primarily compare our method with representative approaches from DyGNNs, OOD generalization, graph OOD generalization, and dynamic graph OOD generalization.
\begin{itemize}[leftmargin=*]
    \item \textbf{DyGNNs}: GCRN \cite{seo2018structured} first combines GCN with RNN to capture spatio-temporal information in graph data. EvolveGCN \cite{pareja2020evolvegcn} employs an RNN to update the parameters of GCN to capture the temporal dynamics in dynamic graphs. DySAT \cite{sankar2020dysat} utilizes joint structural attention and temporal self-attention for dynamic graphs to obtain node embeddings.
    \item \textbf{OOD generalization methods:} IRM \cite{arjovsky2019invariant} targets to minimize the empirical risks across all training domains via learning an invariant predictor. V-REx \cite{krueger2021VREx} minimizes the empirical risk by paying more attention to domains with larger errors. GraphDRO \cite{sagawa2019distributionally} aims to reduce the risk gap across training domains to boost generalization ability of the model under distribution shifts.
    \item \textbf{Graph OOD generalization methods:} DIR \cite{wu2022dir} aims to discover invariant rationale in graphs via interventions on the training distribution to achieve OOD generalization. EERM \cite{wu2022handling} resorts to multiple context explorers that are adversarially trained to maximize the variance of risks from multiple virtual environments.
    \item \textbf{Dynamic Graph OOD generalization methods:} DIDA \cite{zhang2022dynamic} proposes a disentangled spatio-temporal attention network to capture the variant and invariant patterns via causal intervention. SILD \cite{zhang2023spectral} handles distribution shifts in dynamic graphs by capturing and utilizing invariant and variant spectral patterns. EAGLE \cite{yuan2023eagle} introduces a framework for OOD generalization by modeling complex coupled environments and exploiting spatio-temporal invariant patterns. OOD-Linker \cite{tieu2025oodlinker} proposes an error-bounded invariant link selector that can distinguish invariant components to make the model generalizable for different testing scenarios.

\end{itemize}


\subsection{Real-Word Datasets}
We use three real-world datasets. Collab \cite{tang2012cross} is an academic collaboration network with papers published in 16 years; ACT \cite{kumar2019predictingact} describes students behaviors on a MOOC platform within 30 days. Aminer \cite{tang2008arnetminer} is a citation network extracted from different publishers in 17 years. Following the same setting in \cite{zhang2022dynamic}, the challenging inductive future link prediction task is adopted on Collab and ACT, where the model should predict the links in the next time step via exploiting historical graph snapshots. To evaluate the model’s generalization ability under OOD shifts, we assess the model on another dynamic graph with different domains, which is unobserved during training. Following the same setting in \cite{wu2022handling}, we also adopt the challenging inductive node classification task on the Aminer dataset, where the test nodes are strictly unobserved during training, which is more practical and challenging in real-world dynamic graphs. In this dataset, the patterns may vary in different graph snapshots due to the development of deep learning.

\textbf{Collab} \cite{tang2012cross} is an academic collaboration dataset with papers that were published during 1990-2006 (16 graph snapshots). Nodes and edges represent authors and co-authorship, respectively. Based on the co-authored publication, there are five attributes in edges, including “Data Mining”, “Database”, “Medical Informatics”, “Theory” and “Visualization”. We pick “Data Mining” as the shifted attribute. We apply word2vec \cite{mikolov2013word2vec} to extract 32-dimensional node features from paper abstracts. We use 10/1/5 chronological graph snapshots for training, validation, and testing, respectively. The dataset includes 23,035 nodes and 151,790 links in total.

\textbf{ACT} \cite{kumar2019predictingact} describes student actions on a MOOC platform within a month (30 graph snapshots). Nodes represent students or targets of actions, edges represent actions. Considering the attributes of different actions, we apply K-Means \cite{hartigan1979algorithm} to cluster the action features into five categories and randomly select a certain category (the 5th cluster) of edges as the shifted attribute. We assign the features of actions to each student or target and expand the original 4-dimensional features to
32 dimensions by a linear function. We use 20/2/8 chronological graph snapshots for training, validation, and testing, respectively. The dataset includes 20,408 nodes and 202,339 links in total.

\textbf{Aminer} \cite{tang2008arnetminer} is a citation network extracted from DBLP, ACM, MAG, and other sources. We select the top 20 venues, and the task is to predict the venues of the papers. We use word2vec [5] to extract
128-dimensional features from paper abstracts and average to obtain paper features. The distribution shift may come from the out-break of deep learning. We train on papers published between 2001 - 2011, validate on those published in 2012-2014, and test on those published since 2015. The dataset includes 43,141  nodes and 851,527 links in total.

\begin{figure}[t]
\centering
\includegraphics[width=1.0\columnwidth]{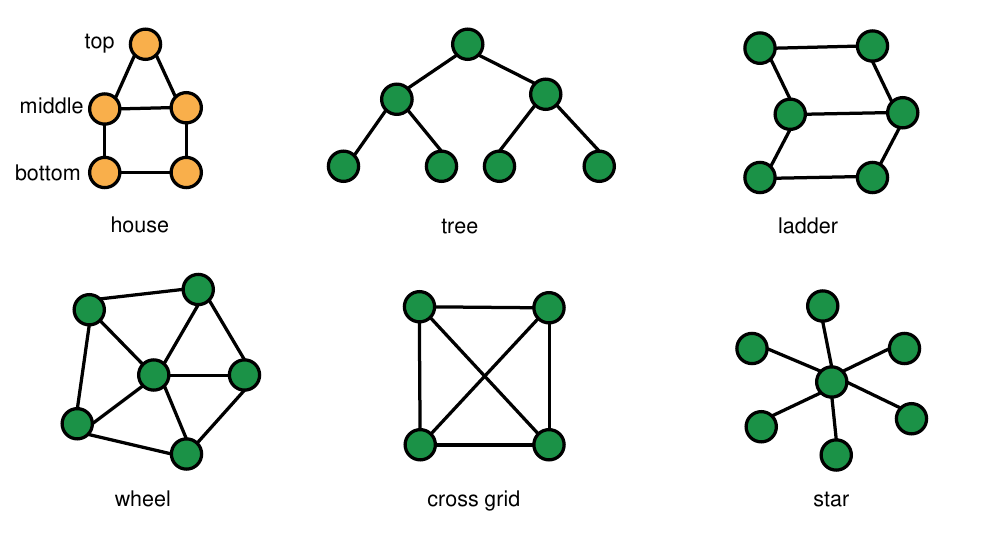} 
\caption{Illustration of the six motif structures in Temporal-Motif, where the 'house' motif is the causal subgraph, and the others are variant subgraphs.}
\label{fig_mptif}    
\end{figure}

\begin{figure*}[ht]
\centering
\subfloat{
\centering
\includegraphics[width=0.185\linewidth]{fig/collab_abl.pdf}
}%
\subfloat{
\centering
\includegraphics[width=0.19\linewidth]{fig/ACT_abl.pdf}
}%
\subfloat{
\centering
\includegraphics[width=0.19\linewidth]{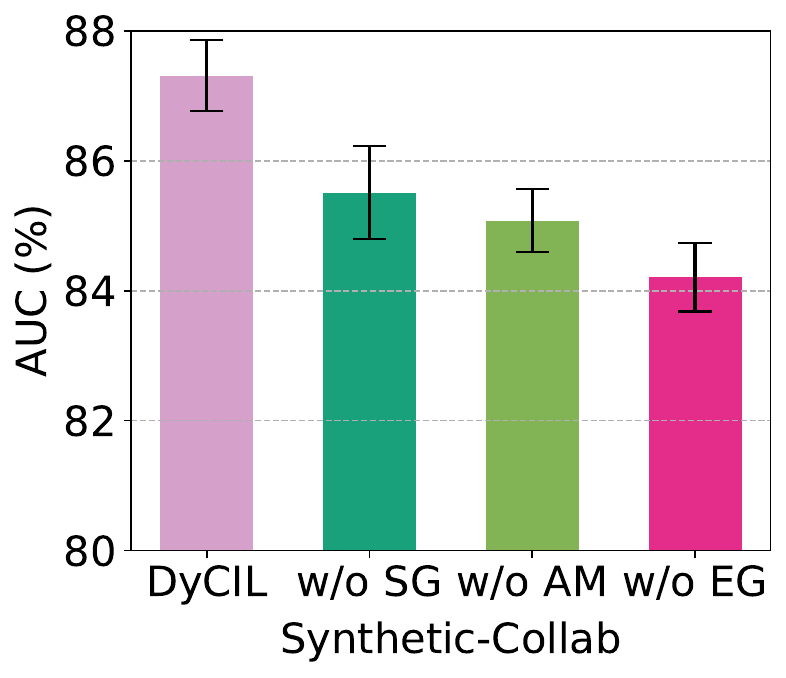}
}%
\subfloat{
\centering
\includegraphics[width=0.19\linewidth]{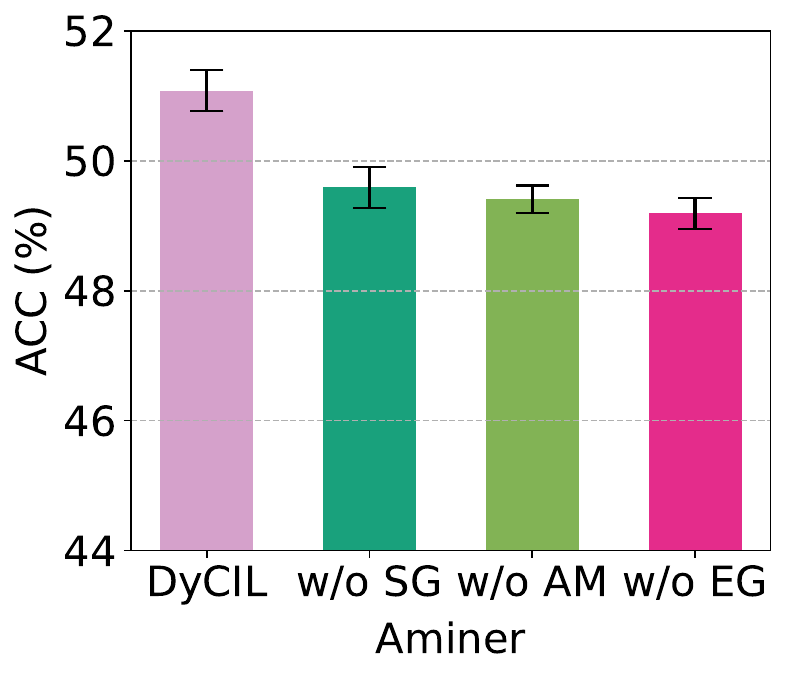}
}%
\subfloat{
\centering
\includegraphics[width=0.19\linewidth]{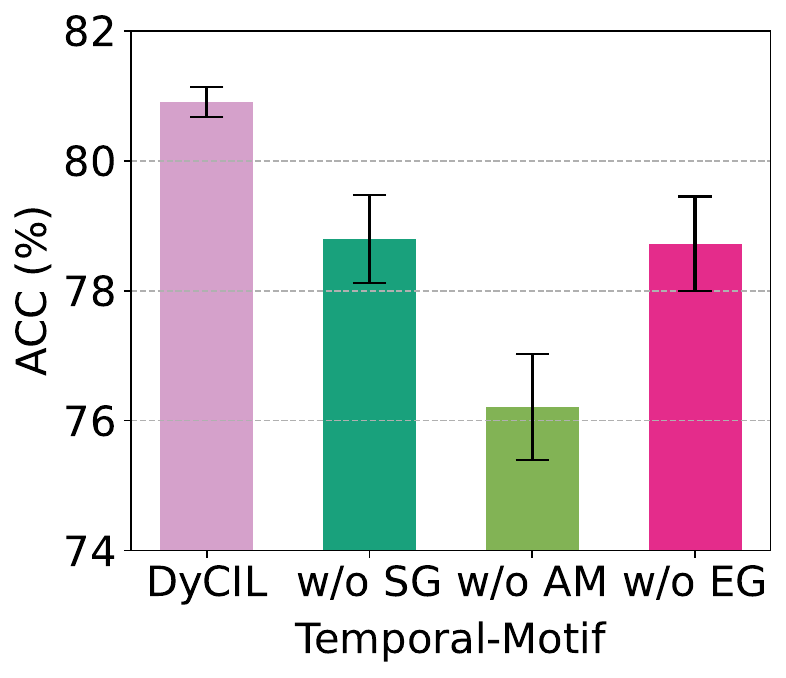}
}%
\centering
\caption{The ablation experiment results.}
\label{figure_ablations}
\end{figure*}

\subsection{Synthetic Datasets}
To further measure the model’s generalization ability under OOD shifts, we synthesize two datasets for node classification and link prediction by introducing manually designed distribution shifts. For link prediction datasets, we introduce variant node features on the original dataset Collab by following \cite{zhang2022dynamic}, where these variant features are generated with spurious correlations in the next timestamps. The degree of OOD shift is determined by a shift level parameter $\bar{p}$. We set $\bar{p}$ to 0.4, 0.6, and 0.8 for training and 0.1 for evaluation, which creates different degrees of OOD shifts between the training data and test data. For the node classification task, we synthesize a new dynamic network called Temporal-Motif by following \cite{ying2019gnnexplainer}. We use the 'house' motif as the causal subgraph and denote other motifs ('tree'-motif, 'wheel'-motif, and so on) as variant subgraphs. In this dataset, the node labels are determined by the 'house' motif solely, and the variant subgraphs in the test data are different from the training data, causing the distribution shift.

\textbf{Synthetic-Collab} \cite{zhang2022dynamic} introduces manual-designed distribution shift on Collab dataset. Denote the original features and structures in Collab as $\mathbf{X}_t^1$ and $\mathbf{A}_t^1$, respectively. We introduce features $\mathbf{X}_t^2$ with a variable correlation with the labels, which are obtained by training the embeddings $\mathbf{X}_t^2 \in \mathbb{R}^{N \times d}$ with the reconstruction loss $\ell(\mathbf{X}_t^2 (\mathbf{X}_t^2)^T, \tilde{\mathbf{A}}_{t+1})$, where $\tilde{\mathbf{A}}_{t+1})$ is the sampled links, and $l$ refers to cross-entropy loss function. In this way, the generated features can have strong correlations with the sampled links. For each timestamp $t$, we uniformly sample $p(t)|\mathcal{E}^{t+1}|$ positive links and $(1 - p(t))|\mathcal{E}^{t+1}|$ negative links in $\mathbf{A}^{t+1}$ and sampling probability $p(t)={\rm clip}(\bar{p} + \sigma \cos(t), 0, 1)$ refers to the intensity of shifts. By controlling the parameter $p$, we can control the correlations of $\mathbf{X}_t$ and labels $\mathbf{A}_{t+1}$ to vary in the training and test stage. Since the model observes the $\mathbf{X}_{t} = [\mathbf{X}_1^{t}||\mathbf{X}_2^{t}]$ simultaneously and the variant features are not marked, the model should discover and get rid of the variant features to handle distribution shifts. Similar to the Collab dataset, we use 10,1,5 chronological graph slices for training, validation, and testing respectively.

\textbf{Temporal-Motif} introduces manually designed distribution shifts for node classification tasks, by following \cite{ying2019gnnexplainer}. We use six different motif structures to synthesize a dynamic motif network, namely the 'house' motif, 'tree' motif, 'ladder' motif, 'wheel' motif, 'cross grid' motif, and 'star' motif. Figure \ref{fig_mptif} shows the structures of these motifs. We designate the 'house' motif as the causal subgraph, meaning that the node labels are determined solely by the 'house' motif. Nodes are assigned to 3 classes based on their structural roles. In a 'house' motif, there are 3 types of roles: the top, middle, and bottom node of the house. corresponding to 3 different classes. For snapshot $\mathcal{G}_t$, the number of 'house' motifs is $25(t+10) + \epsilon$, where $\epsilon$ is a random number between 0 and 10. For training data, we randomly select one structure from the 'tree' motif, 'ladder' motif, and 'wheel' motif as the variant subgraph. The 'cross grid' motif and 'star' motif are used as the variant subgraph $\mathcal{G}_t^{e}$ for validation data and test data, respectively. The variant subgraph $\mathcal{G}_t^{e}$ will be merged with the causal subgraph $\mathcal{G}_t^{c}$, and then $|\mathcal{E}_t^c|$ random perturbation edges will be added to generate the final snapshot $\mathcal{G}_t$. To ensure that the 'house' motif determines the label of each node, we set the number of nodes in the variant subgraph to be less than or equal to the number of nodes in the causal subgraph. We generate the 8-dimensional feature matrix following a Gaussian distribution. This dataset has different variant subgraphs, so the model must capture the causal subgraph across different environments to make accurate predictions.

\begin{figure*}[t]
\centering
\subfloat[Collab]{
\centering
\includegraphics[width=0.23\linewidth]{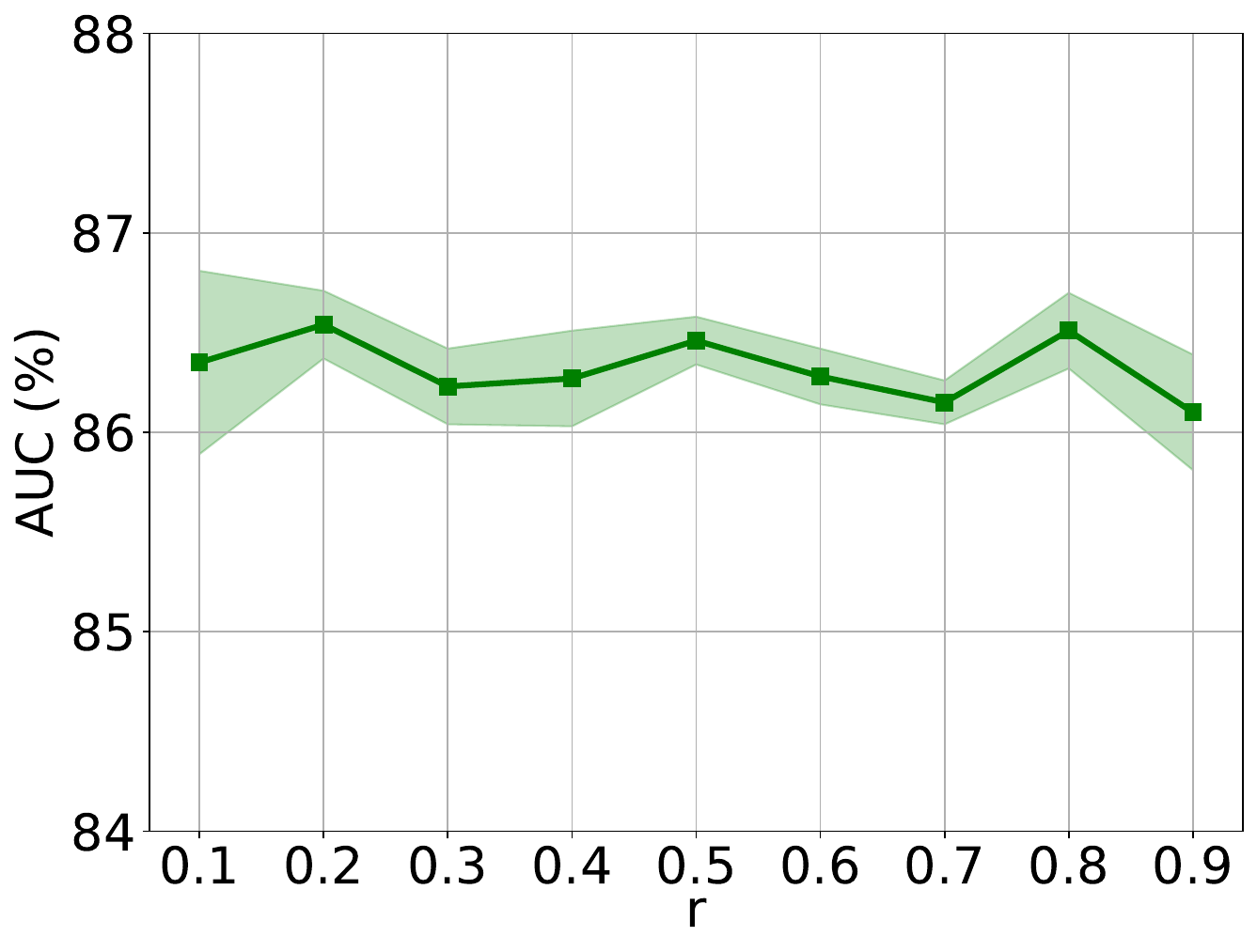}
}%
\subfloat[ACT]{
\centering
\includegraphics[width=0.23\linewidth]{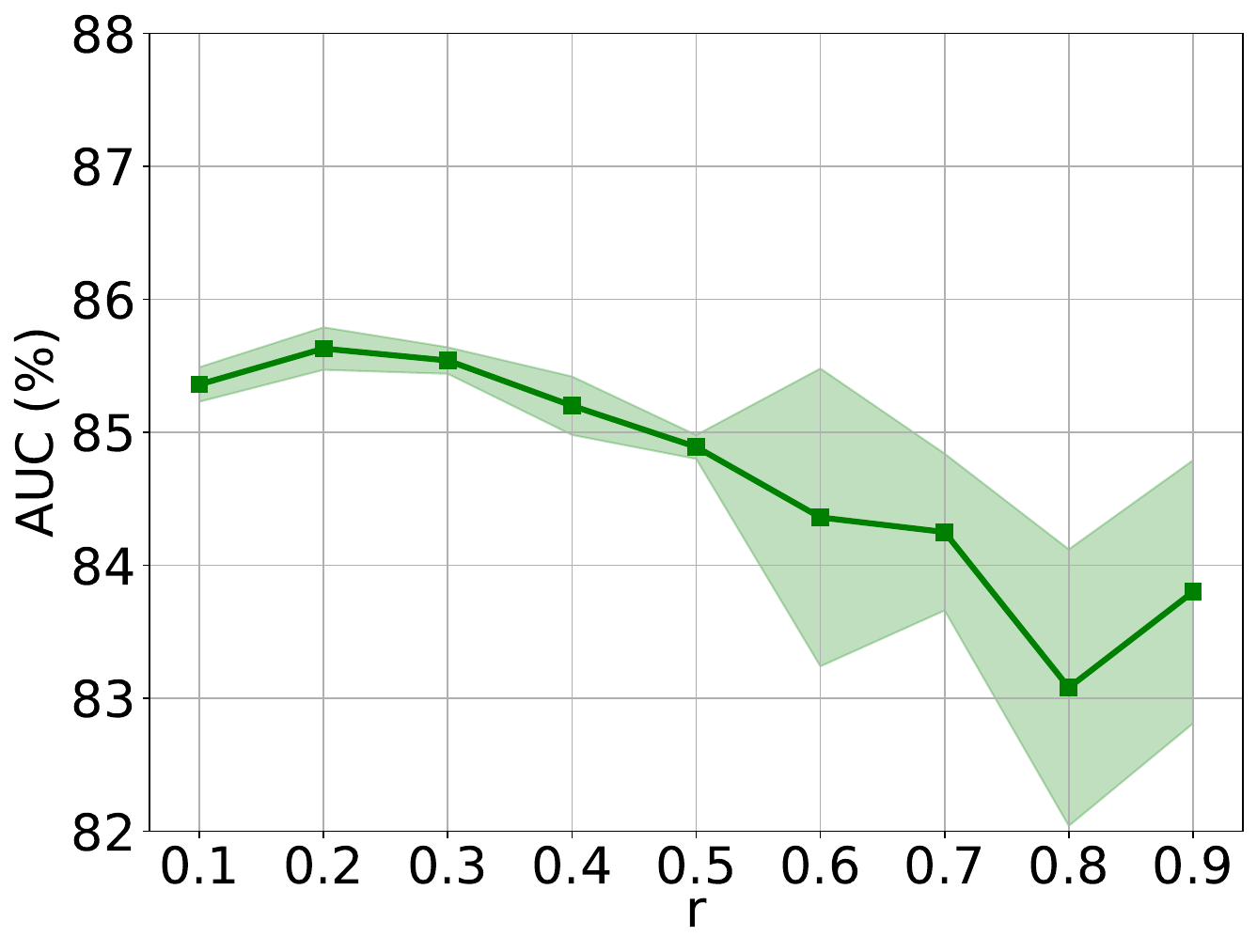}
}%
\subfloat[Aminer]{
\centering
\includegraphics[width=0.23\linewidth]{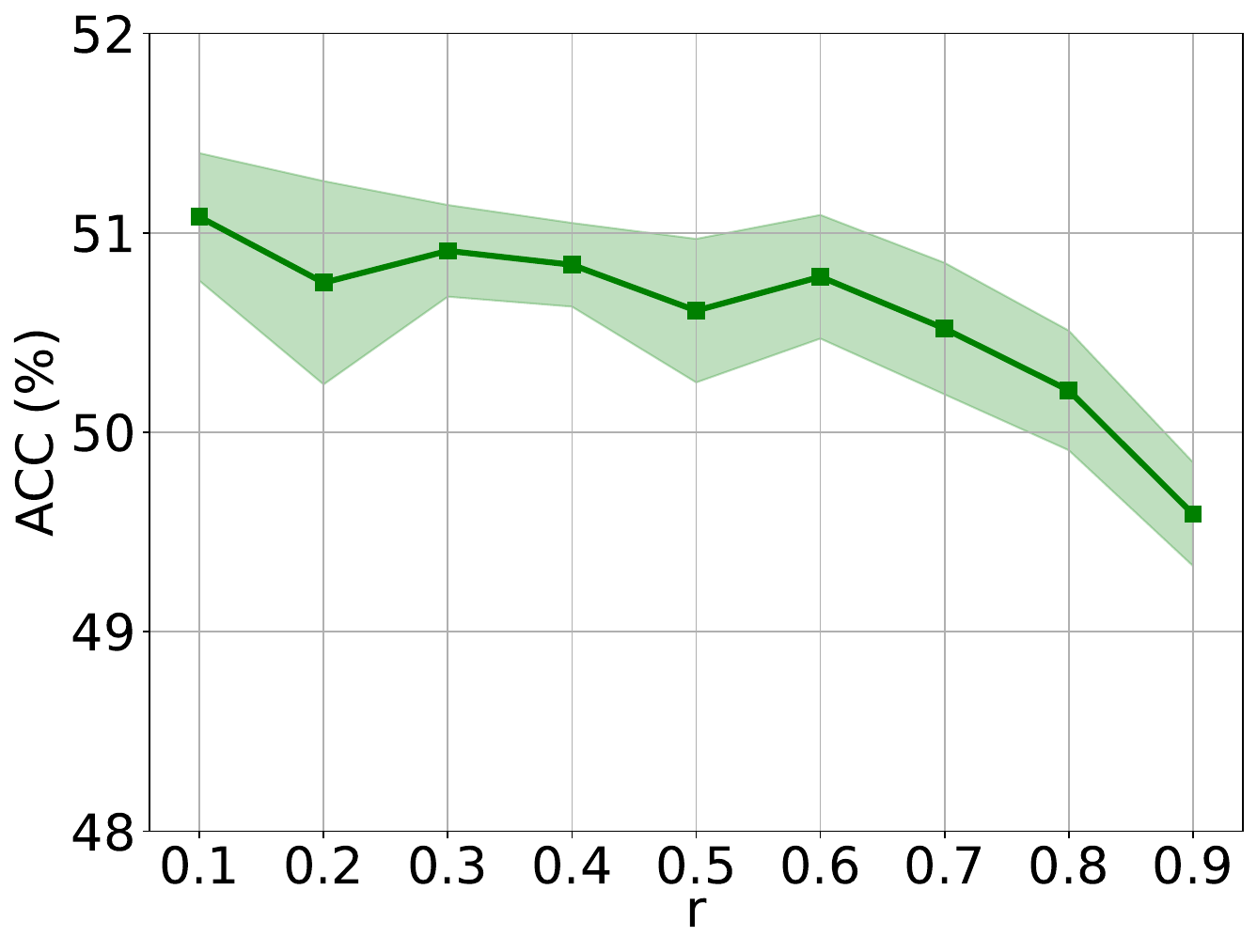}
}%
\subfloat[Temporal-Motif]{
\centering
\includegraphics[width=0.23\linewidth]{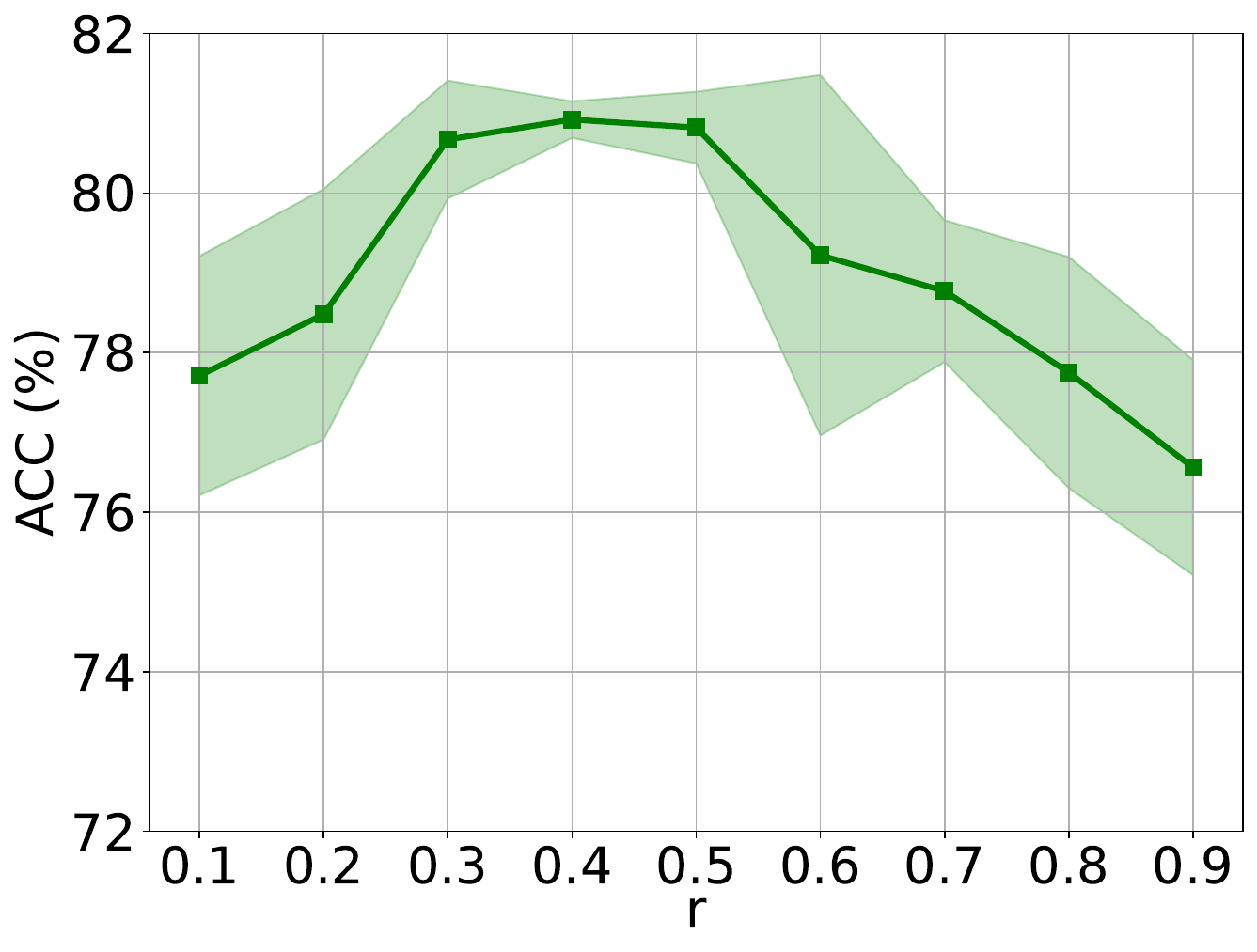}
}%

\centering
\caption{Parameter sensitivity analysis on causal ratio $r$.}
\label{figure_par_sen1}
\end{figure*}

\begin{figure*}[t]
\centering
\subfloat[Collab]{
\centering
\includegraphics[width=0.23\linewidth]{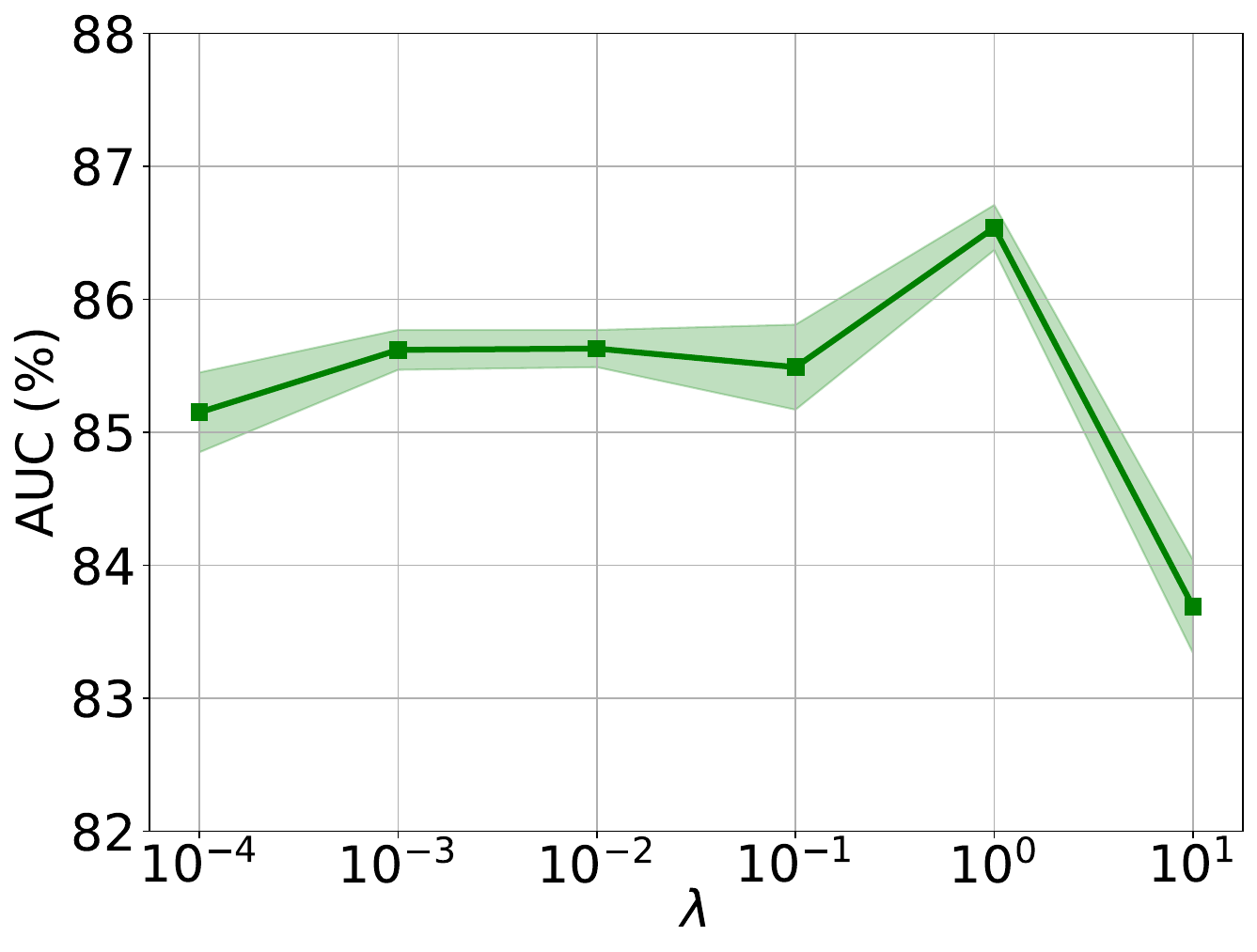}
}%
\subfloat[ACT]{
\centering
\includegraphics[width=0.23\linewidth]{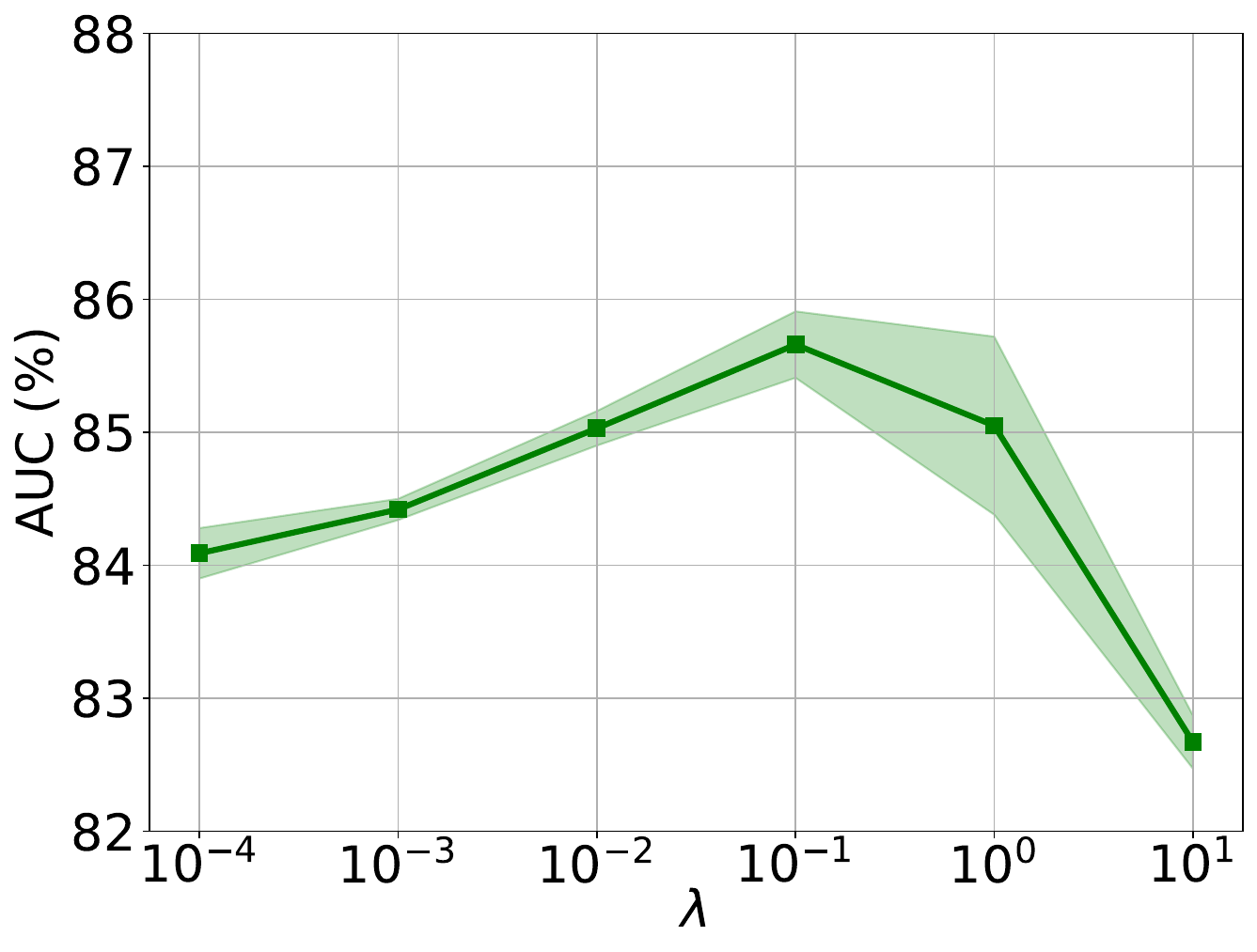}
}%
\subfloat[Aminer]{
\centering
\includegraphics[width=0.23\linewidth]{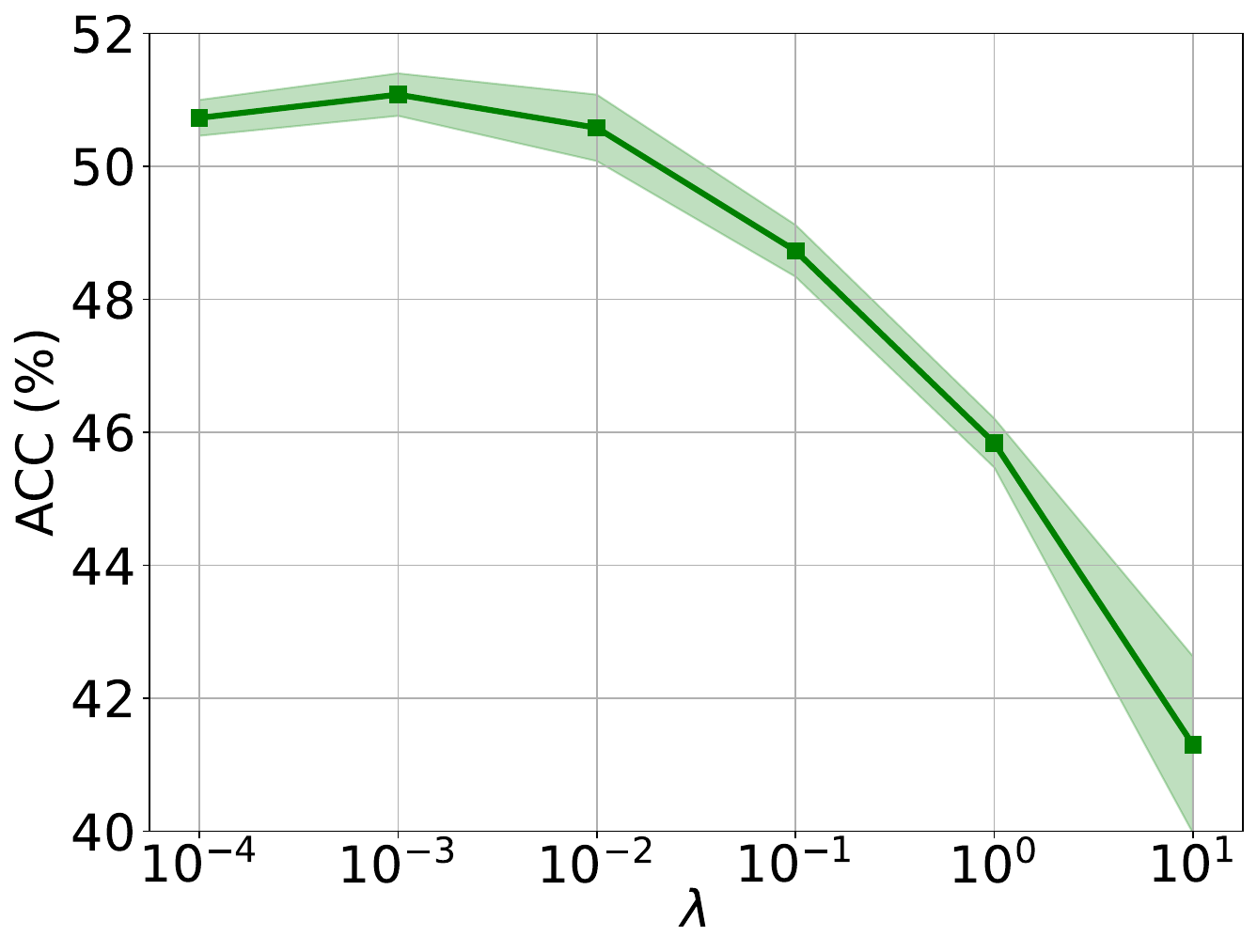}
}%
\subfloat[Temporal-Motif]{
\centering
\includegraphics[width=0.23\linewidth]{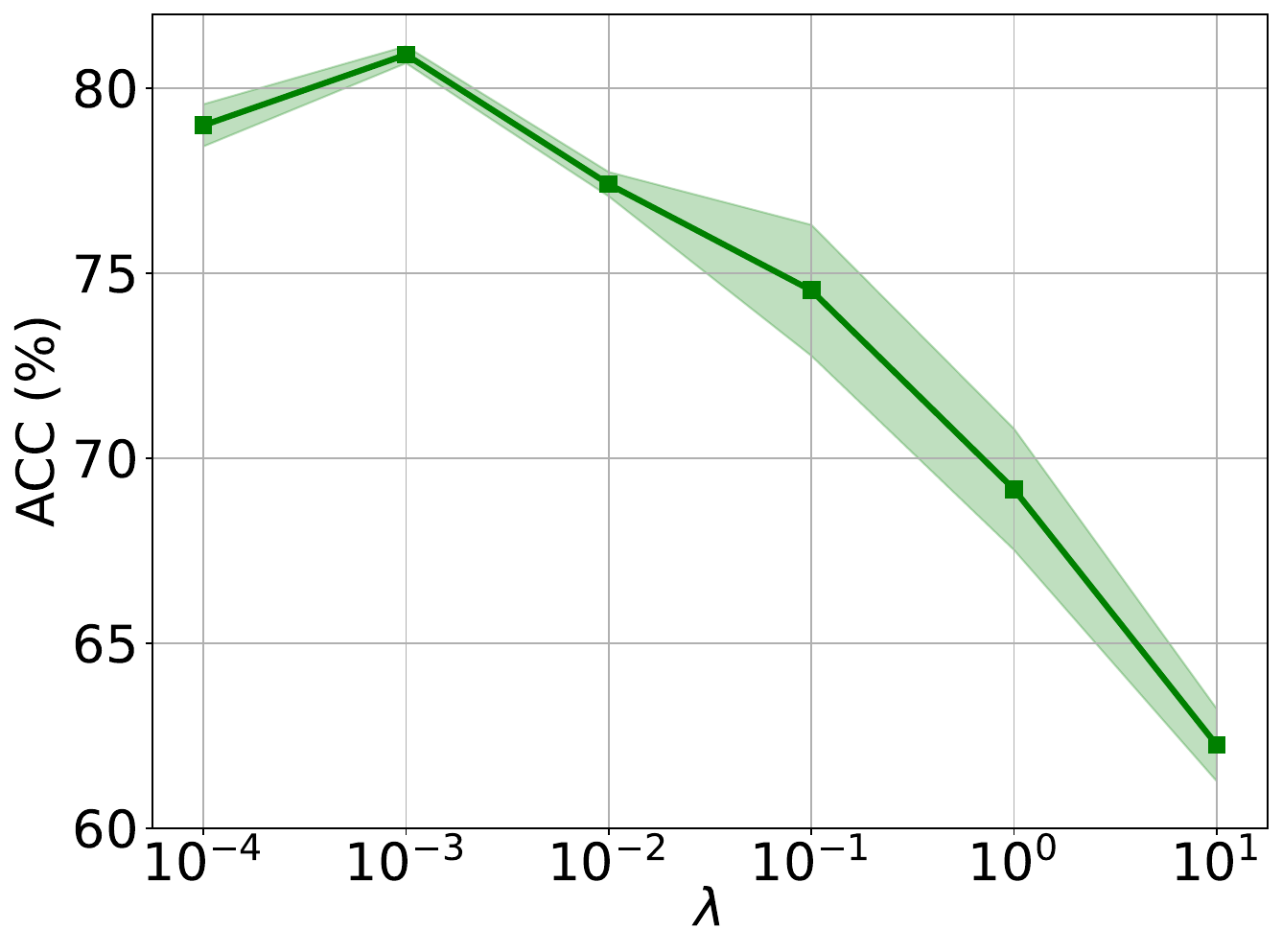}
}%

\centering
\caption{Parameter sensitivity analysis on trade-off weight $\lambda$.}
\label{figure_par_sen2}
\end{figure*}

\section{More Expeiments} \label{app_experiments}

\subsection{Ablation Study}
\begin{itemize}[leftmargin=*]
    \item \textbf{w/o SG}: DyCIL without the dynamic causal subgraph generator, \textit{i.e.}, we utilize the complete dynamic graph as input for the causal-aware spatio-temporal attention and adaptive environment generator.
    \item \textbf{w/o AM}: DyCIL without the causal-aware spatio-temporal attention module, \textit{i.e.}, we utilize DySAT instead of this module to extract invariant spatio-temporal patterns.
    \item \textbf{w/o EG}: DyCIL without adaptive environment generator, \textit{i.e.}, we don't use the generated environment instances sampled from distribution to perform causal interventions.
\end{itemize}

We present the further ablation experiment results in Figure \ref{figure_ablations}. Overall, DyCIL outperforms all its variants across various datasets. Removing any one of the three components leads to a significant performance drop, further validating the effectiveness of each component of our model. The dynamic causal subgraph generator plays a more significant role on the Temporal-Motif dataset, as the causal subgraphs in this dataset are constructed based on motif, allowing our model to generate more accurate subgraphs. In contrast, the causal-aware spatio-temporal attention and environment distribution generator contribute more substantially to performance on real-world datasets, where the evolution patterns are more complex. By capturing the evolution rationale and inferring environment distribution, these components help the model better identify invariant patterns. Specifically, when the dynamic causal subgraph generator is removed, the extracted invariant node embedding and the generated environment instances contain redundant information, leading to an overall performance decline. Removing the causal-aware spatio-temporal attention results in an inability to accurately capture the invariant spatio-temporal patterns of the identified causal dynamic subgraphs. Furthermore, removing the adaptive environment generator hinders the model from leveraging causal interventions to enhance OOD generalization capabilities. In summary, DyCIL jointly optimizes three mutually promoting modules to effectively capture invariant spatio-temporal patterns and achieve OOD generalization under distribution shifts in dynamic graphs.

\subsection{Parameter Analysis}

We perform parameter sensitivity analysis to show
the effect of causal ratio $r$ and trade-off weight $\lambda$. Figure \ref{figure_par_sen1} and Figure \ref{figure_par_sen2} correspond to the experiment results of causal ratio $r$ and trade-off weight $\lambda$. 
\subsubsection{Effect of causal ratio $r$} We can observe that the model performs better when $r$ takes smaller values in most cases. Specifically, DyCIL achieves optimal performance with $r$ set to 0.2, 0.2, 0.1, and 0.4 for the Collab, ACT, Aminer, and Temporal-Motif datasets, respectively. This is because, for most real-world datasets, dynamic graphs contain more spurious correlations while the invariant causal rationale is fewer. A larger $r$ would result in the identified causal dynamic subgraph containing more spurious correlations. In addition, as $r$ increases, the performance of DyCIL tends to decline.  Therefore, a smaller causal ratio (\textless0.4) can help the model learn more informative patio-temporal patterns. In the Temporal-Motif dataset, we set the number of 'house' motifs to account for approximately 40 \% of the dynamic graph such that DyCIL achieves its best performance at $r=0.4$. 
\subsubsection{Effect of trade-off weight $\lambda$}We can observe that as $\lambda$ is too large or too small, the performance of the DyCIL all show varying degrees of decline. A proper $\lambda$ can help the model more effectively capture invariant spatio-temporal patterns in dynamic graphs. Trade-off weight $\lambda$ is related to the complexity of dynamic graphs. When the evolution and interactions in the dynamic graph are more complex, the causal subgraph becomes harder to extract. In such cases, a larger $\lambda$ can make the model focus more on extracting invariant patterns. It shows that $\lambda$ acts as a balance between how DyCIL exploits the patterns and satisfies the invariance constraint.

\begin{figure}[t]
\centering
\subfloat[Collab]{
\centering
\includegraphics[width=0.47\linewidth]{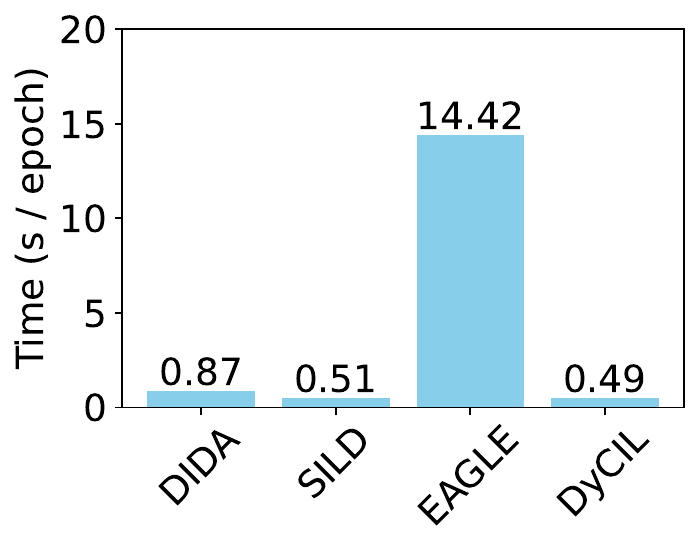}
}%
\subfloat[ACT]{
\centering
\includegraphics[width=0.47\linewidth]{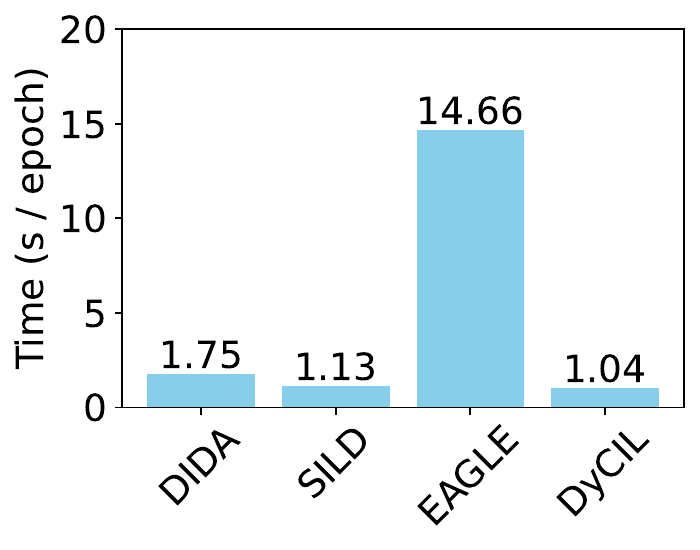}
\label{fig_time} 
}%

\centering
\caption{The training time cost per epoch of DyCIL and other three dynamic graphs OOD methods on Collab and ACT datasets}
\label{figure_time_cost}
\end{figure}


\subsection{Time Cost} \label{time}
We give the computational complexity analysis in Section \ref{app_complex}. Here, we conduct training time cost experiments between DyCIL and other dynamic graphs OOD methods to analyze the scalability and efficiency of DyCIL. We show the result in Figure \ref{figure_time_cost}. Clearly, DyCIL requires less runtime compared to several existing dynamic graph OOD methods. Specifically, EAGLE consumes a significant amount of time due to its reliance on extensive information for inferring environment distributions. In contrast, DyCIL, DIDA, and SILD all have linear computational complexity, placing their time costs on the same scale. However, DyCIL benefits from its three mutually promoting modules, which enhances its invariant spatio-temporal pattern extraction capability without requiring more network layers stacking. This results in lower time overhead. In summary, DyCIL demonstrates superior OOD generalization abilities and maintains sound computational efficiency, showcasing excellent scalability.

\section{Detailed Related work}
\subsection{Dynamic Graph Neural Networks} 
Dynamic graph neural networks have attracted increasing research attention to incorporate temporal dynamics and graph topology features into node representation since most real-world graphs are dynamic \cite{huang2022dlp,huang2023TGB}. Based on the form of dynamic graphs, DyGNN methods can be divided into two main categories: continuous-time DyGNNs and discrete-time DyGNNs \cite{kazemi2020representation}. Continuous-time DyGNNs view a dynamic graph as a flow of edges with specific timestamps \cite{nguyen2018continuous,jin2022neural}. Discrete-time DyGNNs view a dynamic graph as a series of snapshots with timestamps \cite{sankar2020dysat,hajiramezanali2019variational,zhang2022dynamic}.
Continuous-time DyGNNs \cite{rossi2020temporal,xu2020inductive,wang2021inductive,jin2022neural,wu2024feasibility} capture temporal dynamic information by time-encoding techniques and extract spatial structure features by GNN or memory module.
Discrete-time DyGNNs \cite{seo2018structured,hajiramezanali2019variational,pareja2020evolvegcn,sankar2020dysat,yang2021discrete,zhang2023dyted,bai2023hgwavenet} capture spatial-temporal patterns by utilizing GNNs and sequence models to separately capture the structural information and temporal dependencies of dynamic graphs. Although these methods have achieved significant success in the dynamic graph learning domain, most existing methods ignore distribution shifts that commonly exist in dynamic graphs, which makes them fail to generalize to dynamic graph OOD scenarios.

\subsection{Graph Out-of-Distribution Generalization}
Research on graph OOD generalization has attracted significant interest from the graph machine learning community due to the distribution shifts commonly existing in real-world graph data \cite{gui2022good,li2022out}.
Most current mainstream methods can be mainly divided into two types: disentanglement-based graph OOD models and causality-based graph OOD models \cite{li2022out}. Disentanglement-based graph OOD models aim to separate these informative factors from the complex graph environment and map these factors into vector representations to enhance graph OOD generalization performance\cite{yang2020factorizable,liu2020independence,fan2022debiasing,li2022disentangled,wang2024disentangled}. Causality-based graph OOD models inherently capture causal relationships between graph and labels which are invariant and stable across different environment distribution, thereby achieving good OOD generalization \cite{wu2022dir,wu2022handling,li2022learning,chen2022learning,gui2023joint,jia2024graph,wu2024graph}. Although these methods have achieved good generalization in node-level and graph-level tasks on static graphs, they overlook the fact that real-world graphs are dynamic. Therefore, studying the OOD generalization problem in dynamic graphs is more challenging but of high importance as it describes how the real dynamic system interacts and evolves. 

\subsection{Dynamic Graph Out-of-Distribution Generalization}
OOD generalization in dynamic graphs has garnered increasing research interest, with most existing approaches focusing on the perspective of causal invariance \cite{zhang2022dynamic,zhang2023spectral,yuan2023eagle,zhang2023out,zhang2024disentangled}. DIDA \cite{zhang2022dynamic} leverages decoupled dynamic graph attention to extract invariant spatiotemporal patterns, while SILD \cite{zhang2023spectral} approaches the problem from a spectral domain perspective to handle distribution shifts unobservable in the time domain. However, both methods rely on limited observed data to perform causal interventions, which undermines their generalization capabilities in complex OOD scenarios. EAGLE \cite{yuan2023eagle} enhances generalization by inferring latent environment distributions to generate environment instances, but it overlooks the intrinsic evolutionary relationships within dynamic graphs. I-DIDA \cite{zhang2024disentangled} extends DIDA to handle spatio-temporal distribution shifts in sequential recommendation by discovering and utilizing invariant patterns. EpoD \cite{yang2024improving} proposes a self-prompted learning mechanism to infer underlying environment variables and a structural causal model with dynamic subgraphs to capture the effect of environment variable shifts. OOD-Linker \cite{tieu2025oodlinker} leverages the information bottleneck method to extract relevant information from complex data structures and utilizes an error-bound invariant link selector to distinguish between invariant and variant components. However, none of them can address all three challenges simultaneously, and face various limitations in complex OOD scenarios.



\end{document}